\documentclass{article}
\usepackage{graphicx,amsthm} 
\usepackage{subfigure}
\usepackage{wrapfig}
\usepackage{sidecap}
\usepackage{authblk}

\usepackage{multicol}
\usepackage{url}            

\usepackage{algorithm}
\usepackage{algorithmic}
\usepackage{bm}
\usepackage{amsmath}
\usepackage{amssymb}
\usepackage{amsfonts}
\usepackage{latexsym}
\usepackage{multirow,tikz,amsmath,comment}

\usetikzlibrary{arrows,fit,positioning,shapes,snakes}

\let\oldbibliography\thebibliography
\renewcommand{\thebibliography}[1]{%
  \oldbibliography{#1}%
  \setlength{\itemsep}{1pt}%
}

\newdimen\arrowsize
\pgfarrowsdeclare{arcsq}{arcsq}
{
  \arrowsize=0.2pt
  \advance\arrowsize by .5\pgflinewidth
  \pgfarrowsleftextend{-4\arrowsize-.5\pgflinewidth}
  \pgfarrowsrightextend{.5\pgflinewidth}
}
{
  \arrowsize=1.5pt
  \advance\arrowsize by .5\pgflinewidth
  \pgfsetdash{}{0pt} 
  \pgfsetroundjoin   
  \pgfsetroundcap    
  \pgfpathmoveto{\pgfpoint{0\arrowsize}{0\arrowsize}}
  \pgfpatharc{-90}{-140}{4\arrowsize}
  \pgfusepathqstroke
  \pgfpathmoveto{\pgfpointorigin}
  \pgfpatharc{90}{140}{4\arrowsize}
  \pgfusepathqstroke
}

\newtheorem{Theorem}{Theorem}

\newcommand\independent{\protect\mathpalette{\protect\independenT}{\perp}}
\def\independenT#1#2{\mathrel{\rlap{$#1#2$}\mkern2mu{#1#2}}}

\title{Discovery and Visualization of Nonstationary Causal Models}
\author[1,2]{Kun Zhang\thanks{kunz1@andrew.cmu.edu}}
\author[1,2]{Biwei Huang\thanks{biwei.huang@tuebingen.mpg.de}}
\author[3]{Jiji Zhang\thanks{jijizhang@ln.edu.hk}}
\author[2]{Bernhard Sch{\"o}lkopf\thanks{bs@tuebingen.mpg.de}}
\author[1]{Clark Glymour\thanks{cg09@andrew.cmu.edu}}
\affil[1]{Carnegie Mellon University}
\affil[2]{MPI for Intelligent Systems, T{\"u}bingen, Germany}
\affil[3]{Lingnan University, Hong Kong}

\date{}

\begin{document}
\maketitle

\vspace{-1cm}
\begin{abstract}


It is commonplace to encounter nonstationary data, of which the underlying generating process may change over time or across domains. The nonstationarity presents both challenges and opportunities for causal discovery. In this paper we propose a principled framework to handle nonstationarity, and develop methods to address three important questions. First, we propose an enhanced constraint-based method to detect variables whose local mechanisms are nonstationary and recover the skeleton of the causal structure over observed variables. Second, we present a way to determine some causal directions by taking advantage of information carried by changing distributions. Third, we develop a method for visualizing the nonstationarity of local mechanisms. Experimental results on various synthetic and real-world datasets are presented to demonstrate the efficacy of our methods.\end{abstract}

\section{Introduction}
In many fields of empirical sciences and engineering, we would like to obtain causal knowledge for many purposes.  As it is often difficult if not impossible to carry out randomized experiments, inferring causal relations from purely observational data, known as the task of causal discovery, has drawn much attention in several fields including computer science, statistics, philosophy, economics, and neuroscience.  
With the rapid accumulation of huge volumes of data of various types, causal discovery is facing exciting opportunities but also great challenges. One phenomenon such data often feature is that of distribution shift. Distribution shift may occur across domains or over time. For an example of the former kind, consider the problem of remote sensing image classification, which aims to derive land use and land cover information through the process of interpreting and classifying remote sensing imagery. The data collected in different areas and at different times usually have different distributions due to different physical factors related to ground, vegetation, illumination conditions, etc. As an example of the latter kind, the fMRI recordings are usually nonstationary: the causal connections in the brain may change with stimuli, tasks, states, the attention of the subject, etc.  More specifically, it is believed that one of the basic properties of the neural connections in the brain is its time-dependence \cite{DynamicC1}. To these situations many existing approaches to causal discovery fail to apply, as they assume a fixed causal model and hence a fixed joint distribution underlying the observed data. 



In this paper we assume that the underlying causal structure is a directed acyclic graph (DAG), but the mechanisms or parameters associated with the causal structure, or in other words the causal model, may change across domains or over time (we allow mechanisms to change in such a way that some causal links in the structure become vacuous or vanish over some time periods or domains). We aim to develop a principled framework to model such situations as well as practical methods to address these questions:

\begin{itemize}
\item{How to efficiently identify the variables whose local causal mechanisms are nonstationary and recover the skeleton of the causal structure over the observed variables?}

\item{How to take advantage of the information carried by distribution shifts for the purpose of identifying causal directions?}

\item{How to visualize the nonstationarity of those causal mechanisms that change over time or across domain?}
\end{itemize}

This paper is organized as follows. In Section~\ref{Sec:problem} we define and motivate the problem in more detail and review related work.  Section~\ref{Sec:method} proposes an enhanced constraint-based method for recovering the skeleton of the causal structure over the observed variables and identify those variables whose generating processes are nonstationary.  Section~\ref{Sec:unify} develops a method for determining some causal directions by exploiting nonstationarity. Section~\ref{Sec:vis} proposes a way to visualize nonstationarity. Section~\ref{Sec:simul} reports simulations results to test the performance of the proposed causal discovery approach when the ground truth is known.
Finally, we apply the method to some real-world datasets, including financial date and fMRI data, in Section~\ref{Sec:real}.

\section{Problem Definition and Related Work}\label{Sec:problem}

\subsection{Causal Discovery of Fixed Causal Models}
Most causal discovery methods assume that there is a fixed causal model underlying the observed data and aim to estimate it from the data. Classic approaches to causal discovery divide roughly into two types.  In late 1980's and early 1990's, it was noted that under appropriate assumptions, one could recover a Markov equivalence class of the underlying causal structure based on conditional independence relationships among the variables~\cite{Spirtes00,Pearl00}. This gives rise to the constraint-based approach to causal discovery, and the resulting equivalence class may contain multiple DAGs (or other related graphical objects to represent causal structures), which entail the same conditional independence relationships.  The required assumptions include the causal Markov condition and the faithfulness assumption, which entail a correspondence between separation properties in the underlying causal structure and statistical independence properties in the data.  The so-called score-based approach (see, e.g.,~\cite{Chickering02,Heckerman95}) searches for the equivalence class which gives the highest score under some scoring criterion, such as the Bayesian Information Criterion (BIC) or the posterior of the graph given the data.

Another set of approaches is based on restricted functional causal models, which represent the effect as a function of the direct causes together with an independent noise term~\cite{Pearl00}.  Under appropriate assumptions, these approaches are able to identify the whole causal model. More specifically, the causal direction implied by the restricted functional causal model is generically identifiable, in that the model assumptions, such as the independence between the noise and cause, hold only for the true causal direction and are violated for the wrong direction. Examples of such restricted functional causal models include the Linear, Non-Gaussian, Acyclic Model (LiNGAM~\cite{Shimizu06}), the additive noise model~\cite{Hoyer09,Zhang09_additive}, and the post-nonlinear causal model~\cite{Zhang_UAI09}.  The method presented in~\cite{Mooij10_GPI} makes use of a certain type of smoothness of the function in the correct causal direction to distinguish cause from effect, though it does not give explicit identifiability conditions.

\subsection{With Nonstationary Causal Models}
Suppose we are working with a set of observed variables $\mathbf{V} = \{V_i\}_{i=1}^{n}$ and the underlying causal structure over $\mathbf{V}$ is represented by a DAG $G$. For each $V_i$, let $PA^i$ denote the set of parents of $V_i$ in $G$. Suppose at each point in time or in each domain, the joint probability distribution of $\mathbf{V}$ factorizes according to $G$:
\begin{equation} \label{Eq:decomp}
P(\mathbf{V}) = \prod_{i=1}^n P(V_i \,|\, PA^i).
 \end{equation}
We call each $P(V_i\,|\,PA^i)$ a causal module. If there are distribution shifts (i.e., $P(\mathbf{V})$ changes over time or across domains), at least some causal modules $P(V_k\,|\,PA^k)$, $k\in \mathcal{N}$ must change. We call those causal modules {\it nonstationary causal modules}.  Their changes may be due to changes of the involved functional models, causal strengths, noise levels, etc. We assume that those quantities that change over time or cross domains can be written as functions of a time or domain index, and denote by $C$ such an index.

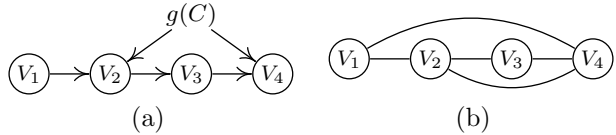
\begin{figure}[htp]
\setlength{\abovecaptionskip}{0pt}
\setlength{\belowcaptionskip}{0.5pt}
  \begin{center}
\setlength{\abovecaptionskip}{-0.2pt}
\setlength{\belowcaptionskip}{0pt}
\begin{center}
\begin{tikzpicture}[scale=.6, line width=0.5pt, inner sep=0.2mm, shorten >=.1pt, shorten <=.1pt]
\draw (0, 0) node(1) [circle, draw] {{\footnotesize\,$V_1$\,}};
  \draw (1.8, 0) node(2) [circle, draw] {{\footnotesize\,$V_2$\,}};
\draw (3.6, 0) node(3) [circle, draw] {{\footnotesize\,$V_3$\,}};
\draw (5.4, 0) node(4) [circle, draw] {{\footnotesize\,$V_4$\,}};
\draw (3.6, 1.3) node(5) {{\footnotesize\,{\small$g(C)$}\,}};
  \draw[-arcsq] (1) -- (2); 
  \draw[-arcsq] (2) -- (3); 
  \draw[-arcsq] (3) -- (4); 
  \draw[-arcsq] (5) -- (2); 
  \draw[-arcsq] (5) -- (4); 
\end{tikzpicture} ~~
\begin{tikzpicture}[scale=.6, line width=0.5pt, inner sep=0.2mm, shorten >=.1pt, shorten <=.1pt]
\draw (0, 0) node(1) [circle, draw] {{\footnotesize\,$V_1$\,}};
  \draw (1.8, 0) node(2) [circle, draw] {{\footnotesize\,$V_2$\,}};
\draw (3.6, 0) node(3) [circle, draw] {{\footnotesize\,$V_3$\,}};
\draw (5.4, 0) node(4) [circle, draw] {{\footnotesize\,$V_4$\,}};
  \draw[-] (1) -- (2); 
  \draw[-] (2) -- (3); 
  \draw[-] (3) -- (4); 
  \draw[-] (1) to[out=30,in=150] (4); 
  \draw[-] (2) to[out=-30,in=-150] (4); 
\end{tikzpicture} \\
(a) ~~~~~~~~~~~~~~~~~~~~~~~~~~~~~~~~(b)
\end{center}
\caption{An illustration on how ignoring changes in the causal model may lead to spurious connections by the constraint-based method. (a) The true causal graph (including confounder $g(C)$). (b) The estimated conditional independence graph on the observed data in the asymptotic case. }
\label{fig:illust_C}
  \end{center}
  \vspace{-10pt}
  \vspace{1pt}
\end{figure}

If the changes in some modules are related, one can treat the situation as if there exists some unobserved quantity (confounder) which influences those modules and, as a consequence, the conditional independence relationships in the distribution-shifted data will be different from those implied by the true causal structure. Therefore, standard constraint-based algorithms such as PC~\cite{Spirtes00,Pearl00} may not be able to reveal the true causal structure.  
As an illustration, suppose that the observed data were generated according to Fig.~\ref{fig:illust_C}(a), where $g(C)$, a function of $C$, is involved in the generating processes for both $V_2$ and $V_4$;  the conditional independence graph for the observed data then contains spurious connections $V_1 - V_4$ and $V_2 - V_4$,  as shown in Fig.~\ref{fig:illust_C}(b), because there is only one conditional independence relationship, $V_3 \independent V_1 \,|\, V_2$.
\begin{figure*}[htp]
\hspace{0cm}\includegraphics[width = 1.0\textwidth]{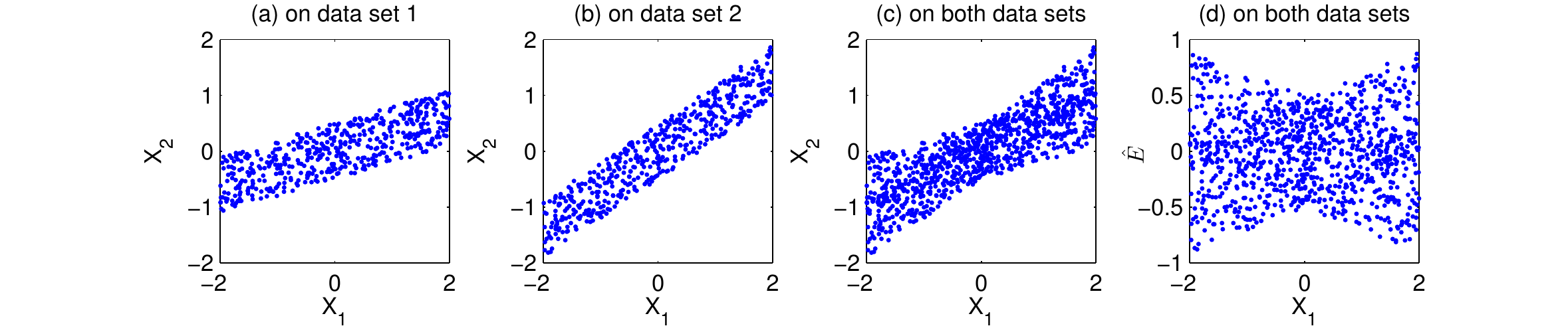} \vspace{-0.7cm}
\caption{An illustration of a failure of using the approach based on functional causal models for causal direction determination when the causal model changes. (a) Scatter plot of $V_1$ and $V_2$ on dataset 1. (b) That on dataset 2. (c) That on merged data (both datasets). (d) The scatter plot of $V_1$ and the estimated regression residual  on merged data.}
\label{fig:illust_SEM}
\end{figure*}

Moreover, when one fits a fixed functional causal model (e.g., a linear, non-Gaussian model~\cite{Shimizu06}) to distribution-shifted data, the estimated noise may not be independent from the cause any more. Consequently, the approach based on restricted functional causal models in general cannot infer the correct causal structure either.  Fig.~\ref{fig:illust_SEM} gives an illustration of this point. Suppose we have two datasets for variables $V_1$ and $V_2$: $V_2$ is generated from $V_1$ according to $V_2 = 0.3V_1 + E$ in the first and according to $V_2 = 0.7V_1 + E$ in the second, and in both datasets $V_1$ and $E$ are mutually independent and follow a uniform distribution. Fig.~\ref{fig:illust_SEM}(a - c) show the scatter plots of $V_1$ and $V_2$ on dataset 1, on dataset 2, and on merged data, respectively. (d) then shows the scatter plot of $V_1$, the cause, and the estimated regression residual on both datasets; they are not independent any more, although on either dataset the regression residual is independent from $V_1$.



To tackle the issue of changing causal models, one may try to find causal models on sliding windows ~\cite{Chronnectome} (for nonstationary data) or for different domains (for data from multiple domains) \textcolor{black}{separately}, and then compare them.  Improved versions include the online changepoint detection method~\cite{Adam07}, the online undirected graph learning~\cite{Talih05}, the locally stationary structure tracker algorithm~\cite{Kummerfeld13}.  Such methods may suffer from high estimation variance due to sample scarcity, large type II errors, and a large number of statistical tests. Some methods aim to estimate the time-varying causal model by making use of certain types of smoothness of the change~\cite{Huang15}, but they do not explicitly locate the nonstationary causal modules. Several methods aim to model time-varying time-delayed causal relations~\cite{Xing10,Song09_DBN}, which can be reduced to online parameter learning because the direction of the causal relations is given (i.e., the past influences the future).  Compared to them, learning changing instantaneous causal relations, with which we are concerned in this paper, is generally more difficult. Moreover, most of these methods assume linear causal models, limiting their applicability to complex problems with nonlinear causal relations.

In contrast, we will develop a nonparametric and computationally efficient method that can identify nonstationary causal modules and recover the causal  skeleton.
We will also show that distribution shifts actually contain useful information for the purpose of determining causal directions and develop practical algorithms accordingly.

\section{Enhanced Constraint-Based Procedure} \label{Sec:method}

\subsection{Assumptions} \label{Sec:assumptions}
As already mentioned, we allow changes in some causal modules to be related, which may be explained by positing unobserved confounders. Intuitively, such confounders may refer to some high-level background variables. For instance, for fMRI data, they may be the subject's attention or unmeasured background stimuli impinging on a subject--scanner noise, random thoughts, physical sensations, etc.; for the stock market, they may be related to economic policies and changes in the ownership among the companies, etc. Thus we do not assume causal sufficiency for the set of observed variables. However, we assume that the confounders, if any, can be written as smooth functions of time or domain index. It follows that at each time or in each domain, the values of these confounders are fixed. We call this a {\it weak causal sufficiency} assumption.

Denote by $\{g_l(C)\}_{l=1}^L$ the set of such confounders (which may be empty). We further assume that for each $V_i$ the local causal process for $V_i$ can be represented by the following structural equation model (SEM):
\begin{equation} \label{Eq:FCM}
V_i = f_i\big(PA^i, \mathbf{g}^i(C), \theta_i(C), \epsilon_i\big),
\end{equation}
where $\mathbf{g}^i(C)\subseteq \{g_l(C)\}_{l=1}^L$ denotes the set of confounders that influence $V_i$, $\theta_i(C)$ denotes the effective parameters in the model that are also assumed to be functions of $C$, and $\epsilon_i$ is a disturbance term that is independent of $C$ and has a non-zero variance (i.e., the model is not deterministic). We also assume that the $\epsilon$'s are mutually independent.

Note that $\{g_l(C)\}_{l=1}^L$ are introduced to account for changes in different causal modules that are not independent. As a result, although $\theta_i(C)$ may also contribute to a change in the causal module for $V_i$, changes to the module for $V_i$ due to $\theta_i(C)$ are independent of changes to the module for $V_j$ due to $\theta_j(C)$, $i\neq j$. In other words, $\theta_i(C)$ is specific to $V_i$ and is independent of $\theta_j(C)$ for $i\neq j$. Note that $\mathbf{g}^i(C)$ and $\theta_i(C)$ can be constant, corresponding to stationary causal modules.

In this paper we treat $C$ as a random variable, and so there is a joint distribution over $\mathbf{V}\cup\{g_l(C)\}_{l=1}^L\cup\{\theta_m(C)\}_{m=1}^n$. We assume that this distribution is Markov and faithful to the graph resulting from the following additions to $G$ (which, recall, is the causal structure over $\mathbf{V}$): add $\{g_l(C)\}_{l=1}^L\cup\{\theta_m(C)\}_{m=1}^n$ to $G$, and for each $i$, add an arrow from each variable in $\mathbf{g}^i(C)$ to $V_i$ and add an arrow from $\theta_i(C)$ to $V_i$. We will refer to this augmented graph as $G^{aug}$. Obviously $G$ is simply the induced subgraph of $G^{aug}$ over $\mathbf{V}$.


\subsection{Detecting Changing Modules and Recovering Causal Skeleton}
In this section we propose a method to detect variables whose modules change and infer the skeleton of $G$. The basic idea is simple: we use the (observed) variable $C$ as a surrogate for the unobserved $\mathbf{V}\cup \{g_l(C)\}_{l=1}^L$, or in other words, we take $C$ to capture \textit{C-specific} information. \footnote{Recall that $C$ may simply be time. Thus in this paper we take time to be a special random variable which follows a uniform distribution over the considered time period, with the corresponding data points evenly sampled at a certain sampling frequency. We realize that this view of time will invite philosophical questions, but for the purpose of this paper, we will set those questions aside. One can regard this stipulation as purely a formal device without substantial implications on time {\it per se}.} We now show that given the assumptions in \ref{Sec:assumptions}, we can apply a constraint-based algorithm to $\mathbf{V}\cup\{ C\}$ to detect variables with changing modules and recover the skeleton of $G$.

\begin{algorithm}
\caption{Detection of Changing Modules and Recovery of Causal Skeleton}
\begin{enumerate}
\item Build a complete undirected graph $U_C$ on the variable set $ \mathbf{V} \cup \{ C\}$.
\item (\textit{Detection of changing modules}) For every $i$, test for the marginal and conditional independence between $V_{i}$ and $C$. If they are independent given a subset of $\{V_k\,|\,k\neq i\}$, remove the edge between $V_{i}$ and $C$ in $ U_C $.
\item (\textit{Reovery of causal skeleton}) For every $i\neq j$, test for the marginal and conditional independence between $V_{i}$ and $V_{j}$. If they are independent given a subset of $\{V_k\,|\,k\neq i, k\neq j\}\cup \{ C\}$, remove the edge between $V_{i}$ and $V_{j}$ in $ U_C $. 
\end{enumerate}
\label{rs_causal_discovery}
\end{algorithm}

The procedure is briefly described in Algorithm 1. It outputs an undirected graph, $U_C$, that contains $C$ as well as $\mathbf{V}$. In Step 2, whether a variable $V_i$ has a changing module is decided by whether $V_i$ and $C$ are independent conditional on some subset of other variables. The justification for one side of this decision is trivial. If $V_i$'s module does not change, that means $P(V_i \,|\, PA^i)$ remains the same for every value of $C$, and so $V_i \independent C \, | \, PA^i$. Thus, if $V_i$ and $C$ are not independent conditional on any subset of other variables, $V_i$'s module changes with $C$, which is represented by an edge between $V_i$ and $C$. Conversely, we assume that if $V_i$'s module changes, which entails that $V_i$ and $C$ are not independent given $PA^i$, then $V_i$ and $C$ are not independent given any other subset of $\mathbf{V}\backslash \{ V_i\}$. If this assumption does not hold, then we only claim to detect some (but not necessarily all) variables with changing modules.

Step 3 aims to discover the skeleton of the causal structure over $\mathbf{V}$. Its (asymptotic) correctness is justified by the following theorem:
\begin{Theorem} \label{theo1}
Given the assumptions made in Section~\ref{Sec:assumptions}, for every $V_i, V_j\in \mathbf{V}$, $V_i$ and $V_j$ are not adjacent in $G$ if and only if they are independent conditional on some subset of $\{V_k\,|\,k\neq i, k\neq j\}\cup \{ C\}$.
\end{Theorem}
\begin{proof}
Before getting to the main argument, let us establish some implications of the SEMs Eq.~\ref{Eq:FCM} and the assumptions in Section~\ref{Sec:assumptions}. Since the structure is assumed to be acyclic or recursive, according to Eq.~\ref{Eq:FCM}, all variables $V_i$ can be written as a function of $\{g_l(C)\}_{l=1}^L \cup \{\theta_m(C)\}_{m=1}^n$ and $\{\epsilon_m\}_{m=1}^n$. As a consequence, the probability distribution of $\mathbf{V}$ at each value of $C$ is determined by the distribution of $\epsilon_1, ..., \epsilon_n$, and the values of $\{g_l(C)\}_{l=1}^L \cup \{\theta_m(C)\}_{m=1}^n$. In other words, $p(\mathbf{V}|C)$ is determined by $\prod_{i=1}^n p(\epsilon_i)$ (for $\epsilon_1, ..., \epsilon_n$ are mutually independent), and $\{g_l(C)\}_{l=1}^L \cup \{\theta_m(C)\}_{m=1}^n$, where $p(\cdot)$ denotes the probability density or mass function. For any $V_i$, $V_j$, and $\mathbf{V}^{ij}\subseteq \{V_k\,|\,k\neq i, k\neq j\}$, because $p(V_i, V_j \,|\,  \mathbf{V}^{ij}, C)$ is determined by $p(\mathbf{V}|C)$, it is also determined by $\prod_{i=1}^n p(\epsilon_i)$ and $\{g_l(C)\}_{l=1}^L \cup \{\theta_m(C)\}_{m=1}^n$. Since $\prod_{i=1}^n p(\epsilon_i)$ does not change with $C$, we have
\begin{flalign} \nonumber
&p(V_i, V_j \,|\,  \mathbf{V}^{ij} \cup \{g_l(C)\}_{l=1}^L \cup \{\theta_m(C)\}_{m=1}^n \cup \{C\})  \\ \label{Eq:Cind1}
=& p(V_i, V_j \,|\,  \mathbf{V}^{ij} \cup \{g_l(C)\}_{l=1}^L \cup \{\theta_m(C)\}_{m=1}^n).
\end{flalign}
That is,
\begin{equation} \label{Eq:Cind}
C \independent (V_i, V_j) \,|\,  \mathbf{V}^{ij} \cup \{g_l(C)\}_{l=1}^L \cup \{\theta_m(C)\}_{m=1}^n.
\end{equation}
By the weak union property of conditional independence, it follows that
\begin{equation} \label{Eqc}
C \independent V_j \,|\,\{V_i\}\cup \mathbf{V}^{ij} \cup \{g_l(C)\}_{l=1}^L \cup \{\theta_m(C)\}_{m=1}^n.
\end{equation}

We are now ready to prove the theorem. Let $V_i, V_j$ be any two variables in $\mathbf{V}$.
First, suppose that $V_i$ and $V_j$ are not adjacent in $G$. Then they are not adjacent in $G^{aug}$, which recall is the graph that incorporates $\{g_l(C)\}_{l=1}^L \cup \{\theta_m(C)\}_{m=1}^n$. It follows that there is a set $\mathbf{V}^{ij}\subseteq \{V_k\,|\,k\neq i, k\neq j\}$ such that $\mathbf{V}^{ij} \cup \{g_l(C)\}_{l=1}^L \cup \{\theta_m(C)\}_{m=1}^n$ d-separates $V_i$ from $V_j$. Since the joint distribution over $\mathbf{V} \cup \{g_l(C)\}_{l=1}^L \cup \{\theta_m(C)\}_{m=1}^n$ is assumed to be Markov to $G^{aug}$, we have
\begin{equation} \label{Eqd}
V_i \independent V_j \,|\,  \mathbf{V}^{ij} \cup \{g_l(C)\}_{l=1}^L \cup \{\theta_m(C)\}_{m=1}^n.
\end{equation}
Because all $g_l(c)$ and $\theta_m(C)$ are deterministic functions of $C$, we have $p(V_i, V_j \,|\,  \mathbf{V}^{ij} \cup \{C\}) = p(V_i, V_j \,|\,  \mathbf{V}^{ij} \cup \{g_l(C)\}_{l=1}^L \cup \{\theta_m(C)\}_{m=1}^n \cup \{C\})$.

According to~\cite{Madiman08} , Eqs.~\ref{Eqd} and~\ref{Eq:Cind} imply
$V_i \independent (C, V_j) \,|\,  \mathbf{V}^{ij} \cup \{g_l(C)\}_{l=1}^L \cup \{\theta_m(C)\}_{m=1}^n$.
By the weak union property of conditional independence, it follows that
$V_i \independent V_j \,|\,  \mathbf{V}^{ij} \cup \{g_l(C)\}_{l=1}^L \cup \{\theta_m(C)\}_{m=1}^n \cup \{C\}$.
As all $g_l(C)$ and $\theta_m(C)$ are deterministic functions of $C$, it follows that
$V_i \independent V_j \,|\,  \mathbf{V}^{ij} \cup \{C\}$.
In other words, $V_i$ and $V_j$ are conditionally independent given a subset of $\{V_k\,|\,k\neq i, k\neq j\}\cup \{C\}$.

Conversely, suppose $V_i$ and $V_j$ are conditionally independent given a subset $\mathbf{S}$ of $\{V_k\,|\,k\neq i, k\neq j\}\cup \{C\}$. We show that $V_i$ and $V_j$ are not adjacent in $G$, or equivalently, that they are not adjacent in $G^{aug}$. There are two possible cases to consider:

\begin{itemize}
\item Suppose $\mathbf{S}$ does not contain $C$. Then since the joint distribution over $\mathbf{V} \cup \{g_l(C)\}_{l=1}^L \cup \{\theta_m(C)\}_{m=1}^n$ is assumed to be Faithful to $G^{aug}$, $V_i$ and $V_j$ are not adjacent in $G^{aug}$, and hence not adjacent in $G$.

\item Otherwise, $\mathbf{S}=\mathbf{V}^{ij}\cup \{C\}$ for some $\mathbf{V}^{ij}\subseteq \{V_k\,|\,k\neq i, k\neq j\}$. That is,
\begin{flalign} \label{tmp0}
&V_i \independent V_j \,|\,  \mathbf{V}^{ij} \cup \{C\},\textrm{ or } \\ \nonumber
& p(V_i, V_j \,|\,  \mathbf{V}^{ij} \cup \{C\}) = p(V_i \,|\,  \mathbf{V}^{ij} \cup \{C\}) p(V_j \,|\,  \mathbf{V}^{ij} \cup \{C\}).
\end{flalign}

According to Eq.~\ref{Eq:Cind1}, and also noting that $\{g_l(C)\}_{l=1}^L \cup \{\theta_m(C)\}_{m=1}^n$ is a deterministic function of $C$, we have
\begin{equation}\label{tmp1}
p(V_i, V_j \,|\,  \mathbf{V}^{ij}\cup \{C\}) = p(V_i, V_j \,|\,  \mathbf{V}^{ij} \cup \{g_l(C)\}_{l=1}^L \cup \{\theta_m(C)\}_{m=1}^n),
\end{equation}
which also implies
\begin{flalign} \label{tmp2}
&p(V_i \,|\, \mathbf{V}^{ij}\cup \{C\}) = p(V_i \,|\,  \mathbf{V}^{ij} \cup \{g_l(C)\}_{l=1}^L \cup \{\theta_m(C)\}_{m=1}^n), \\ \label{tmp3}
& p(V_j \,|\,  \mathbf{V}^{ij}\cup  \{C\}) = p(V_j \,|\,  \mathbf{V}^{ij} \cup \{g_l(C)\}_{l=1}^L \cup \{\theta_m(C)\}_{m=1}^n).
\end{flalign}
Substituting Eqs.~\ref{tmp1} -~\ref{tmp3} into Eq.~\ref{tmp0} gives
\begin{flalign}
&p(V_i, V_j \,|\,  \mathbf{V}^{ij} \cup \{g_l(C)\}_{l=1}^L \cup \{\theta_m(C)\}_{m=1}^n) \\ \nonumber
=& p(V_i \,|\,  \mathbf{V}^{ij} \cup \{g_l(C)\}_{l=1}^L \cup \{\theta_m(C)\}_{m=1}^n) p(V_j \,|\,  \mathbf{V}^{ij} \cup \{g_l(C)\}_{l=1}^L \cup \{\theta_m(C)\}_{m=1}^n).
\end{flalign}
That is,
$$V_i \independent V_j \,|\,  \mathbf{V}^{ij} \cup \{g_l(C)\}_{l=1}^L \cup \{\theta_m(C)\}_{m=1}^n.$$
Again, by the Faithfulness assumption on $G^{aug}$, this implies that $V_i$ and $V_j$ are not adjacent in $G^{aug}$ and hence are not adjacent in $G$.
\end{itemize}

Therefore, $V_i$ are $V_j$ are not adjacent in $G$ if and only if they are conditionally independent given some subset of $\{V_k\,|\,k\neq i, k\neq j\}\cup \{C\}$.


\end{proof}


In the above procedure, it is crucial to use a general, nonparametric conditional independence test, for how variables depend on $C$ is unkown and usually very nonlinear. In this work, we use the kernel-based conditional independence test (KCI-test~\cite{Zhang11_KCI}) to capture the dependence on $ C $ in a nonparametric way. By contrast, if we use, for example, tests of vanishing partial correlations, as is widely used in the neuroscience community, the proposed method will not work well.


\section{An Advantage of Nonstationarity in Determination of Causal Direction} \label{Sec:unify}
We now show that using the additional variable $C$ as a surrogate not only allows us to infer the skeleton of the causal structure, but also facilitates the determination of some causal directions. Let us call those variables that are adjacent to $C$ in the output of Algorithm 1 ``$C$-specific variables'', which are actually the effects of nonstationary causal modules. For each $C$-specific variable $V_k$, it is possible to determine the direction of every edge incident to $V_k$, or in other words, it is possible to infer $PA^k$. Let $V_l$ be any variable adjacent to $V_k$ in the output of Algorithm 1. There are two possible cases to consider:
\begin{enumerate}

\item $V_l$ is not adjacent to $C$. Then $C- V_k - V_l$ forms an unshielded triple in the skeleton. For practical purposes, we can take the direction between $C$ and $V_k$ as $C \rightarrow V_k$ (though we do not claim $C$ to be a cause in any substantial sense). Then we can use the standard orientation rules for unshielded triples to orient the edge between $V_k$ and $V_l$ ~\cite{Spirtes00,Pearl00}: if $V_l$ and $C$ are independent given a set of variables excluding $V_k$, then the triple is a V-structure, and we have $V_k \leftarrow V_l$.  Otherwise, if $V_l$ and $C$ are independent given a set of variables including $V_k$, then the triple is not a V-structure, and we have $V_k \rightarrow V_l$.

\item $V_l$ is also adjacent to $C$. This case is more complex than Case 1, but it is still possible to identify the causal direction between $V_k$ and $V_l$, based on the principle that $P(\texttt{cause})$ and $P(\texttt{effect}\,|\,\texttt{cause})$ change independently; a heuristic method is given in Section~\ref{Sec:general}. 
\end{enumerate}

The procedure in Case 1 contains the methods proposed in~\cite{Hoover90,Tian01} for causal discovery from changes as special cases, which may also be interpreted as special cases of the principle underlying the method for Case 2:  
if one of $P(\texttt{cause})$ and $P(\texttt{effect}\,|\,\texttt{cause})$ changes while the other remains invariant, they are clearly independent.

\subsection{Inference of the Causal Direction between Variables with Changing Modules} \label{Sec:general}

We now develop a heuristic method to deal with Case 2 above. For simplicity, let us start with the two variable case: suppose $V_1$ and $V_2$ are adjacent and are both adjacent to $C$ (and not adjacent to any other variable). We aim to identify the causal direction between them, which, without loss of generality, we suppose to be $V_1\rightarrow V_2$. The guiding idea is that nonstationarity may carry information that confirms ``independence'' of causal modules, which, in the simple case we are considering, is the ``independence'' between $P(V_1)$ and $P(V_2|V_1)$. If $P(V_1)$ and $P(V_2|V_1)$ are ``independent'' but $P(V_2)$ and $P(V_1|V_2)$ are not, then the causal direction is inferred to be from $V_1$ to $V_2$. The idea that causal modules are ``independent'' is not new, but in a stationary situation where each module is fixed, such independence is very difficult, if not impossible, to test. By contrast, in the situation we are considering presently, both $P(V_1)$ and $P(V_2|V_1)$ are nonstationary, and we can try to measure the extent to which variation in $P(V_1)$ and variation in $P(V_2)$ are dependent (and similarly for $P(V_2)$ and $P(V_1|V_2)$). This is the sense in which nonstationarity actually helps in the inference of causal directions, and as far as we know, this is the first time that such an advantage is exploited in the case where both $P(\texttt{cause})$ and $P(\texttt{effect}\,|\,\texttt{cause})$ change.

We now derive a method along this line. Note that although both of $V_1$ and $V_2$ are adjacent to $C$, there does not necessarily exist a confounder. Fig.~\ref{fig:illust_direct}(a) shows the case where the involved changing parameters, $\theta_1(C)$ and $\theta_2(C)$ are independent, i.e., $P(V1;\theta_1)$ and $P(V2\,|\,V_1;\theta_2)$ change independently. (We dropped the argument $C$ in $\theta_1$ and $\theta_2$ to simplify notations.)

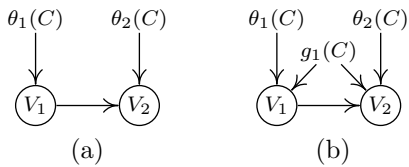
\begin{figure}[htp]
 \begin{center}
\hspace{-.3cm}
\begin{tikzpicture}[scale=.6, line width=0.5pt, inner sep=0.2mm, shorten >=.1pt, shorten <=.1pt]
\draw (.5, 0) node(1) [circle, draw] {{\footnotesize\,$V_1$\,}};
  \draw (2.8, 0) node(2) [circle, draw] {{\footnotesize\,$V_2$\,}};
  \draw (.5, 1.9) node(3)  {{\footnotesize\,$\theta_1(C)$\,}};
    \draw (2.8, 1.9) node(4)  {{\footnotesize\,$\theta_2(C)$\,}};
  \draw[-arcsq] (1) -- (2); 
  \draw[-arcsq] (3) -- (1); 
    \draw[-arcsq] (4) -- (2); 
\end{tikzpicture}~~~~~~~
\begin{tikzpicture}[scale=.6, line width=0.5pt, inner sep=0.2mm, shorten >=.1pt, shorten <=.1pt]
\draw (.5, 0) node(1) [circle, draw] {{\footnotesize\,$V_1$\,}};
  \draw (2.8, 0) node(2) [circle, draw] {{\footnotesize\,$V_2$\,}};
  \draw (.5, 1.9) node(3)  {{\footnotesize\,$\theta_1(C)$\,}};
    \draw (2.8, 1.9) node(4)  {{\footnotesize\,$\theta_2(C)$\,}};
        \draw (1.65, 1.2) node(5)  {{\footnotesize\,$g_1(C)$\,}};
  \draw[-arcsq] (1) -- (2); 
  \draw[-arcsq] (3) -- (1); 
    \draw[-arcsq] (4) -- (2); 
    \draw[-arcsq] (5) -- (1); 
    \draw[-arcsq] (5) -- (2); 
\end{tikzpicture}
\\(a)~~~~~~~~~~~~~~~~~~~~~~~~(b)~~\\
 \end{center}
\caption{Two possible situations where $V_1 \rightarrow V_2$ and both $V_1$ and $V_2$ are adjacent to $C$. (a)  $\theta_1(C) \independent \theta_2(C)$.  (b) In addition to the changing parameters, there is a confounder $g_1(C)$ underlying $V_1$ and $V_2$.} \label{fig:illust_direct}
\end{figure} 

For the reverse direction, one can  decompose the joint distribution of $(V_1,V_2)$ according to
\begin{equation} \label{Eq:reverse}
P(V_1,V_2;\theta_1', \theta_2') = P(V_2;\theta_2')P(V_1\,|\,V_2;\theta_1'),
\end{equation}
where $\theta_1'$ and $\theta_2'$ are assumed to be sufficient for the corresponding distribution terms. Generally speaking, $\theta_1'$ and $\theta_2'$ are not independent, because they are determined jointly by both $\theta_1$ and $\theta_2$. We assume that this is the case, and identify the direction between $V_1$ and $V_2$ based on this assumption.

Now we face two problems. First, how can we compare the dependence between $\theta_1$ and $\theta_2$ and that between between $\theta'_1$ and $\theta'_2$?  Second, in our nonparametric setting, we do not really have such parameters. How can we compare  the dependence based on the given data?  

For the first problem, we make use of the following measures of contributions from the parameters. The total contribution (in a way analogous to causal effect; see~\cite{effect13}) from $\theta'_1$ and $\theta'_2$ to $(V_1, V_2)$ can be measured with mutual information:
\begin{flalign}\nonumber
&\mathcal{S}_{(\theta'_1,\theta'_2) \rightarrow (V_1, V_2)} = I\big((\theta'_1,\theta'_2); (V_1, V_2)\big)  \\ \nonumber
=& I(\theta'_2; V_2) + I(\theta'_1; V_1 \,|\, V_2) + I(\theta'_2; V_1\,|\, \theta_1', V_2) \\ \label{Eq:MI}
=& I(\theta'_2; V_2) + I(\theta'_1; V_1 \,|\, V_2),
\end{flalign}
where the second equality holds because of the chain rule, and the last one because the sufficiency of $\theta_1'$ for $P(V_1\,|\,V_2;\theta_1')$ implies $\theta'_2 \independent V_1\,|\, \theta_1', V_2$. Eq.~\ref{Eq:MI} involves the regular mutual information and conditional mutual information.

Since $\theta'_1$ and $\theta'_2$ are dependent, their individual contributions to $(V_1,V_2)$ are redundant.  Below we calculate the individual contributions. The contribution from $\theta_2'$ to $V_2$ is
$ \mathcal{S}_{\theta'_2 \rightarrow V_2} =  I(\theta'_2; V_2).$
The contribution from $\theta_1'$ to $V_1$ has been derived in~\cite{effect13}:
$ \mathcal{S}_{\theta'_1 \rightarrow V_1} =  \mathbb{E} \Big[\log \frac{P(V_1\,|\,V_2,\theta_1')}{\int P(V_1\,|\,V_2,\tilde{\theta}_1') P(\tilde{\theta}_1') \textrm{d}\tilde{\theta}_1' }\Big]$,
where $\tilde{\theta}_1'$ is an independent copy of ${\theta}'_1$.
As a consequence, the redundancy in the contributions from  $\theta'_1$ and $\theta'_2$ is
\begin{flalign} \nonumber
\Delta_{V_2 \rightarrow V_1} &= \mathcal{S}_{\theta'_2 \rightarrow V_2} + \mathcal{S}_{\theta'_1 \rightarrow V_1} - \mathcal{S}_{(\theta'_1,\theta'_2) \rightarrow (V_1, V_2)} \\ \nonumber
& = \mathbb{E}\Big[\log \frac{P(V_1\,|\,V_2)}{\int P(V_1\,|\,V_2,\tilde{\theta}_1') P(\tilde{\theta}_1') \textrm{d}\tilde{\theta}_1'} \Big] = \mathbb{E}\Big[ \log \frac{P(V_1\,|\,V_2)}{ \mathbb{E}_{\tilde{\theta}_1'}P(V_1\,|\,V_2,\tilde{\theta}_1') } \Big].
\end{flalign}
$\Delta_{V_2 \rightarrow V_1}$ is always non-negative because it is a Kullback-Leibler divergence. One can verify that if $\theta_1' \independent \theta_2'$, which implies $\theta_1' \independent V_2$, we have $\int P(V_1\,|\,V_2,\tilde{\theta}_1') P(\tilde{\theta}_1') \textrm{d}\tilde{\theta}_1' = \int P(V_1\,|\,V_2,\tilde{\theta}_1') P(\tilde{\theta}_1'\,|\,V_2) \textrm{d}\tilde{\theta}_1' = P(V_1\,|\,V_2)$, leading to $\Delta_{V_2 \rightarrow V_1} = 0$.

$\Delta_{V_2 \rightarrow V_1}$ provides a way to measure the dependence between $\theta_1'$ and $\theta_2'$. Regarding the second problem mentioned above, since we do not have parametric models, we propose to estimate $\Delta_{V_2 \rightarrow V_1}$ from the data by:
\begin{equation} \label{Eq:emp}
\hat{\Delta}_{V_2 \rightarrow V_1} = \Big\langle \log \frac{\bar{P}(V_1\,|\,V_2)}{ \langle\hat{P}(V_1\,|\,V_2) \rangle} \Big\rangle,
\end{equation}
where $\langle\cdot\rangle$ denotes the sample average, $\bar{P}(V_1\,|\,V_2)$ is the empirical estimate of ${P}(V_1\,|\,V_2)$ on all data points, and $\langle\hat{P}(V_1\,|\,V_2)\rangle$ denotes the sample average of $\hat{P}(V_1\,|\,V_2)$, the estimate of ${P}(V_1\,|\,V_2)$ at each time (or in each domain).  In our implementation, we used kernel density estimation (KDE) on all data points to estimate $\bar{P}(V_1\,|\,V_2)$, and used KDE on sliding windows (or in each domain) to estimate $\hat{P}(V_1\,|\,V_2)$. We take the direction for which $\hat{\Delta}$ is smaller to be the causal direction.

If there is a confounder $g_1(C)$ underlying $V_1$ and $V_2$, as shown in Fig.~\ref{fig:illust_direct}(b), we conjecture that the above approach still works if the influences from $g_1(C)$ is not very strong, for the following reason: for the correct direction, $\hat{\Delta}$ measures the influence from the confounder; for the wrong direction, it measures the influence from the confounder and the dependence in the ``parameters" caused by the wrong causal direction.  A future line of research is to seek a more rigorous theoretical justification of this method.
When there are more than two variables which are connected to $C$ and inter-connected, we try all possible causal structures and choose the one that minimizes the total $\hat{\Delta}$ value, i.e., $\sum_{i:PA^i\neq \emptyset} \hat{\Delta}_{PA^i \rightarrow V_i}$.



\section{Kernel Nonstationarity Visualization of Causal Modules} \label{Sec:vis}


It is informative to determine for which variable the causal model (data-generating process), or $P(V_i\,|\,PA^i)$, changes. But usually it is not enough -- one often wants to interpret the pattern of the changes, find what causes the changes, and understand the causal process in more detail. To achieve so, it is necessary to discover how the causal model changes, i.e., where the changes occur and how fast it changes, and visualize the changes. Although the changes occur in the conditional distribution $P(V_i\,|\,PA^i)$, usually it is not straightforward to see the properties of the changes by directly looking at the distribution itself. A low-dimensional representation of the changes is needed.

In the parametric case, if we know which parameters of the causal model $PA^i \rightarrow V_i$ are changing, which could be the mean of a root cause, the coefficients in a linear SEM, etc., then we can estimate such parameters for different values of $C$ and see how they change. However, such knowledge is usually not available, and more importantly, for the sake of flexibility we often model the causal processes nonparametrically. Therefore, it is desirable to develop a general nonparametric procedure for nonstationarity visualization of changing causal modules.

Note that changes in $P({V_i\,|\,PA^i})$ are irrelevant to changes in $P(PA^i)$, and accordingly, they are not necessarily the same as changes in the joint distribution $P({V_i, PA^i})$. (If $V_i$ is a root cause, $PA^i$ is an empty set, and $P({V_i\,|\,PA^i})$ reduces to the marginal distribution $P(V_i)$.) 
We aim to find a mapping of $P({V_i\,|\,PA^i})$ which captures its nonstationarity:
\begin{equation} \label{Eq:def_lambda}
\lambda_i(C) = h_i(P({V_i\,|\,PA^i, C})).
\end{equation}
We call $\lambda_i(C)$ the nonstationarity encapsulator for $P({V_i\,|\,PA^i, C})$.  This formulation is rather general: any identifiable parameters in $P({V_i\,|\,PA^i, C})$ can be expressed this way, and in the nonparametric case, $\lambda_i(C)$ can be seen as a statistic to summarize changes in $P({V_i\,|\,PA^i, C})$ along with $C$.
If $P({V_i\,|\,PA^i, C})$ does not change along with $C$, then $\lambda_i(C)$ remains constant. Otherwise, 
$\lambda_i(C)$ is intended to capture the variability of $P({V_i\,|\,PA^i, C})$ across different values of $C$.

Now there are two problems to solve. One is given only observed data, not the conditional distribution, how to represent $\lambda_i(C)$ in Eq.~\ref{Eq:def_lambda} conveniently.  The other is what criterion and method to use to enable $\lambda_i(C)$ to capture the variability in the conditional distribution along with $C$.
We tackle the above two problems by making use of kernels~\cite{Scholkopf02}, and accordingly propose a method called kernel nonstationarity visualization (KNV) of causal modules.

\subsection{Using Kernel Embedding of Conditional Probabilities}
We use the kernel embedding of conditional distributions~\cite{Song09_embedding} instead of the original conditional distributions. Suppose we have kernels $k^{(1)}_X$ and $k^{(1)}_Y$ for variables $X$ and $Y$, with the corresponding Reproducing Kernel Hilbert Spaces (RKHS) $\mathcal{H}^{(1)}_X$ and $\mathcal{H}^{(1)}_Y$, respectively.
Given conditional distribution $P(Y|X)$, its kernel embedding can be seen as an operator mapping from $\mathcal{H}^{(1)}_X$ to $\mathcal{H}^{(1)}_Y$, defined as
$\mathcal{U}_{Y|X} = \mathcal{C}_{YX} \mathcal{C}^{-1}_{XX}$, where $\mathcal{C}_{YX}$ and $\mathcal{C}_{XX}$ denote the (uncentered) cross-covariance and covariance operators, respectively~\cite{Fukumizu04dimensionalityreduction}.  The empirical
estimate of $\mathcal{U}_{Y|X}$ is $\hat{\mathcal{U}}_{Y|X} = \boldsymbol{\Psi}_Y (K_X + \beta I)^{-1} \boldsymbol{\Psi}_X^\intercal$, where $\beta$ is a regularization parameter (set to 0.05 in our experiments), and $\boldsymbol{\Psi}_Y$, $\boldsymbol{\Psi}_X$, and $K_X$ are the feature matrix on $Y$, feature matrix on $X$, and the kernel matrix on $X$, respectively~\cite{Song09_embedding}. We use the Gaussian kernel for $k^{(1)}_X$  and $k^{(1)}_Y$ with kernel width $\sigma_1$, and  $\hat{\mathcal{U}}_{Y|X}$ encodes the information of $P(Y\,|\,X)$ on the given data.

In our problem, we need consider the kernel conditional distribution embedding of $P(V_i\,|\,PA^i)$ {\it for each value of $C$}. If $C$ is a domain index, for each value of $C$ we have a dataset of $(V_i\,|\,PA^i)$. If $C$ is a time index, we use a sliding window to find the data corresponding to $C=c$, by using the data of $(V_i, PA^i)$ in the window of length $L$ centered at $c$.
 As we shall see later, It is possible to avoid directly calculating the empirical estimate of the embedding,  but we need the following (``cross") kernel (or Gram) matrices:
$K_{V_i}({c,c'})$ is the ``cross" kernel matrix between the values of $V_i$ corresponding to $C=c$ and those corresponding to $C=c'$, and similarly for $K_{PA^i}({c,c'})$.



\subsection{Nonstationary Encapsulator Extraction by Eigenvalue Decomposition}

Next, in principle, we use the estimated kernel embedding of conditional distributions, $\hat{\mathcal{U}}_{V_i\,|\,PA^i,C=c}$, as input, and aim to find $\hat{\lambda}_i(c)$ as a (nonlinear) mapping of $\hat{\mathcal{U}}_{V_i\,|\,PA^i, C=c}$, to capture its variability across different $c$. This can be readily achieved by exploiting some nonlinear principle component analysis (PCA) techniques, and here we adopted kernel principal component analysis problem (KPCA)~\cite{Scholkopf98}, for its nice formulation and computational efficiency. KPCA  computes principal components in high-dimensional feature spaces of the input. In our case, for each $c$ the input,  $\hat{\mathcal{U}}_{V_i\,|\,PA^i,C=c}$, is a matrix.  We can stack it into a long vector, and then represent $\hat{\lambda}_i(c)$ by making use of a second kernel, $k^{(2)}$ (which is usually different from $k^{(1)}$), as required by KPCA.  Denote by the corresponding Gram matrix by $M$, whose $(c,c')$th entry is, $M(c,c') \triangleq k^{(2)}\big(\hat{\mathcal{U}}_{V_i\,|\,PA^i,C=c}, \hat{\mathcal{U}}_{V_i\,|\,PA^i,C=c'}\big)$.  Calculating $\hat{\mathcal{U}}_{V_i\,|\,PA^i,C=c}$ involves the empirical kernel maps of $V_i$ and $PA^i$; below we show that we can directly find $M$ without explicitly making use of empirical kernel maps.

If we use a linear kernel for $k^{(2)}$, the $(c,c')$th entry of $M$ is \footnote{When $PA^i$ is an empty set, $P(V_i\,|\,PA^i)$ reduces to $P(V_i)$. In this case we use the embedding of $P(V_i)$, $\mu_{V_i} \triangleq \mathbb{E}_{P(V_i)}[\psi(V_i)]$, whose empirical estimate is the sample mean of $\psi(V_i)$ on the sample. Here $\psi(\cdot)$ denotes the feature map. Accordingly, $M^l(c,c')$ reduces to $\frac{1}{n_c n_{c'}}\mathbf{1}_{n_{c'}}^\intercal K_{V_i}(c',c)\mathbf{1}_{n_c}$, where $n_c$ and $n_{c'}$ are the sizes of the data corresponding to $C=c$ and $C=c'$, respectively, and $\mathbf{1}_{n_c}$ is the vector of 1's of length $n_c$.}
\begin{flalign} \nonumber
&M^l(c,c') = \textrm{Tr}\big[\hat{\mathcal{U}}^\intercal_{V_i\,|\,PA^i,C=c}\hat{\mathcal{U}}_{V_i\,|\,PA^i,C=c'}\big] \\  \nonumber
=& \textrm{Tr}\big[ \boldsymbol{\Psi}_{PA^i}(c) \big(K_{PA^i}({c,c}) + \beta I\big)^{-1} \boldsymbol{\Psi}_{V_i}^{\intercal}(c) \boldsymbol{\Psi}_{V_i}({c'})
 \big(K_{PA^i}({c',c'}) + \beta I \big)^{-1}
 \boldsymbol{\Psi}_{PA^i}(c') \big]\\ \label{Eq:Gl}
=& \textrm{Tr}\big[ K_{V_i}(c',c) \big(K_{PA^i}(c,c) + \beta I \big)^{-1} K_{PA^i}(c,c')  \big(K_{PA^i}(c',c') + \beta I \big)^{-1}
   \big].
\end{flalign}
If $k^{(2)}$ is a Gaussian kernel with kernel width $\sigma_2$, we have
\begin{flalign}  \nonumber
 M^\mathcal{G}(c,c') &= \textrm{exp}\big(- \frac{||\hat{\mathcal{U}}_{V_i\,|\,PA^i}(c) - \hat{\mathcal{U}}_{V_i\,|\,PA^i}(c')||_F^2}{2\sigma_2^2} \big)  \\ \label{Eq:Gg}
&= \textrm{exp}\big(-\frac{
 M^l(c,c) +M^l(c',c') - 2M^l(c',c)}{2\sigma_2^2}
 \big),
\end{flalign}
where $||\cdot||_F$ denotes the Frobenius norm.

Finally, $\hat{\lambda}_i(C)$ can be found by performing eigenvalue decomposition on the above Gram matrix, $M^l$ or $M^g$; for details please see~\cite{Scholkopf98}. Algorithm 2 summarizes the proposed KNR method. There are several hyperparameters to set. In our experiments, we set the kernel width $\sigma_1^2$ (for $k^{(1)}$) and $\sigma_2^2$ (for $k^{(2)}$) to the median distance between points in the sample, as in~\cite{Gretton07}. We kept the window length $L=100$.

\begin{algorithm}
\caption{KNV of Causal Models}
\begin{enumerate}
\item For possible values $c$ and $c'$, calculate $K_{V_i}({c,c'})$ and $K_{PA^i}({c,c'})$ with kernel $k^{(1)}$. If $C$ is a time index, they can be obtained by extracting corresponding entries of the kernel matrices $K_{V_i}$ and $K_{PA^i}$ on the whole data.

\item Calculate Gram matrix $M$ with kernel $k^{(2)}$ (see Eq.~\ref{Eq:Gl} for linear kernels and Eq.~\ref{Eq:Gg} for Gaussian kernels).  \

\item Find $\hat{\lambda}_i(C)$ by directly feeding Gram matrix $M$ to KPCA. That is, perform eigenvalue decomposition on $M$ to find the nonlinear principal components $\hat{\lambda}_i(C)$, as in Section 4.1 of~\cite{Scholkopf98}.
\end{enumerate}
\label{KNV}
\end{algorithm}

\section{Experimental Results on Simulated Data}
\label{Sec:simul}

\subsection{A Toy Example}
We generated synthetic data according to the SEMs specified in Fig. \ref{simulated1}. More specifically, the exogenous input to $V_1$, the causal strength from $V_3$ to $V_5$ (the coefficient $f_3$ in the structural equation for $V_5$), and the noise variance in the equation for $V_4$ are time varying; the changing parameters were represented by sinusoid or cosine functions of $T$. We used different periodic levels ($w=5,10,20,30$) of the varying components, as well as different sample sizes ($N=600, 1000$). In each setting, we run 10 replications, with both our enhanced constraint-based method (Algorithm 1, with the time index for $C$) and the original constraint-based method;  we used the SGS search procedure \cite{SGS93} and kernel-based conditional independence test \cite{Zhang11_KCI}. 

\begin{figure}[ht]
\setlength{\abovecaptionskip}{-0pt}
\setlength{\belowcaptionskip}{-0pt}
\vspace{-3pt}
 \centering

\small
\begin{equation*}\hspace{-0.25cm}
    \left \{
    \begin{aligned}
        V_1&=f_1 \cdot E_0 + E_1,\\
        V_2 & =  \text{sin}(V_1^2)-0.2 V_1+E_2, \\
      V_3 & =0.5 \text{cos}(V_1) + E_3, \\
      V_4 & =\text{sin}(V_2+V_3) + 0.2V_2 + f_2 \cdot E_4, \\
      V_5 & =f_3 \cdot \text{tanh}(V_3) + 0.2V_3 + E_5, \\
      V_6 & =0.5 (V_2+V_5) + E_6.
    \end{aligned}
    \right.~~
\boxed{
\small
  \begin{aligned}
    & f_1 = \text{sin} (w \cdot \frac{t}{N})\\
    & f_2 = 0.8 \text{sin} (w \cdot(\frac{t}{N}+\frac{1}{2}) ),\\
    & f_3 = 1.5 \text{cos} (w \cdot(\frac{t}{N}+\frac{1}{2}) ),\\
    & \text{with } t=1,\cdots, N.\\
    & e_0 \sim U[0,1], \\
    & e_i \sim U[-0.3,0.3],\\
    & \text{with } i=1,\cdots,6.
  \end{aligned}
  }
\end{equation*}
 \caption{The SEMs according to which we generated the simulated data. The input to $V_1$, the noise variance to $V_4$, and the causal strength from $V_3$ to $V_5$ are time varying, represented by $f_1$, $f_2$ and $f_3$, respectively. We tried different $w$, and different sample sizes $N$. }
\label{simulated1}
\end{figure}

\begin{figure}[ht]
\centering
\setlength{\abovecaptionskip}{-3pt}
\setlength{\belowcaptionskip}{-2pt}
 \centering
\includegraphics[width=.8\textwidth,height = 4cm,trim={1.4cm .5cm 2cm .3cm},clip]{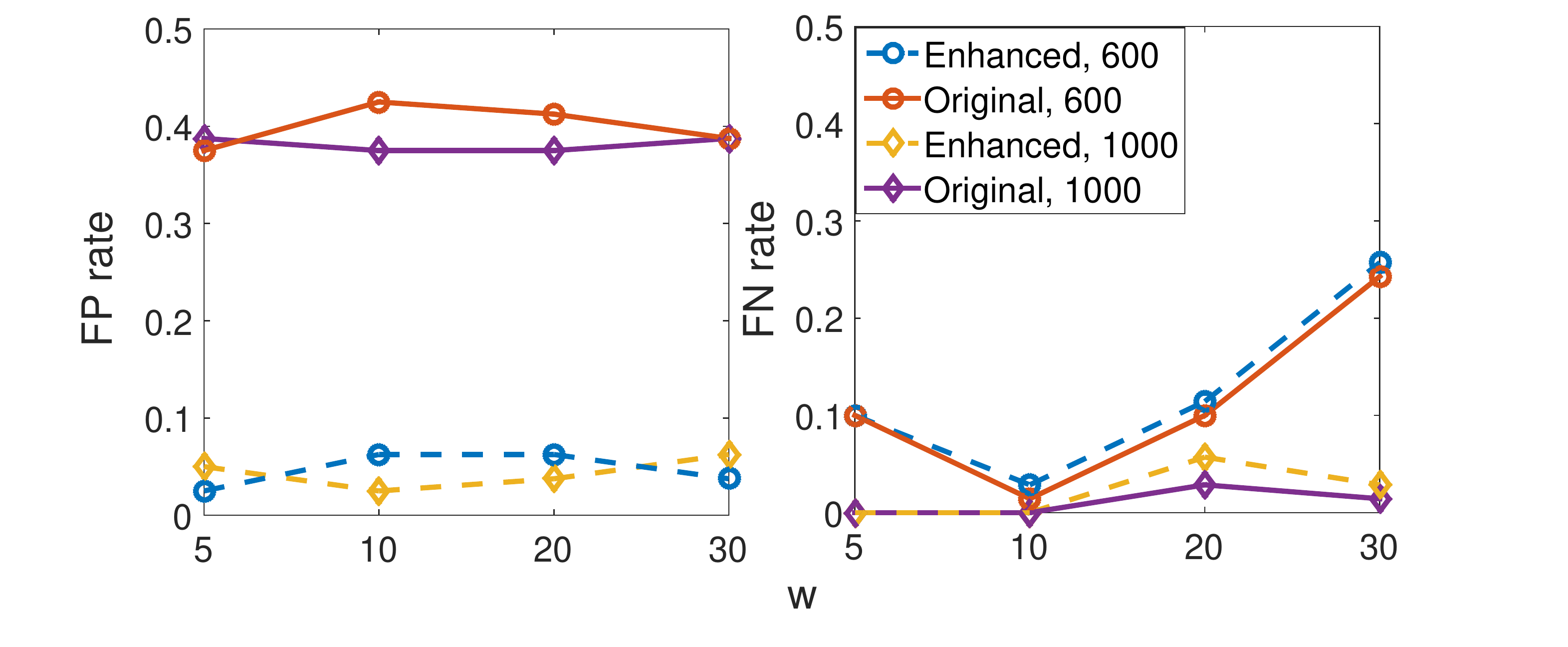}
 \caption{The estimated FP rate and FN rate with $w=5,10,20,30$ and $N=600,1000$ by both our enhanced constraint-based method and the original constraint-based method.}
\label{simulated2}
\end{figure}


Fig. \ref{simulated2} shows the False Positive (FP) rate and the False Negative (FN) rate of the discovered adjacencies between the $V$'s at significance level $0.05$. It is obvious that compared to the original method, our method effectively reduces the number of spurious connections, i.e., edges $(V_1, V_4)$, $(V_1, V_5)$ and $(V_4, V_5)$, in all the settings. The FN rate only very slightly increases.  As $w$ increases, the FP rate stays stable, and the FN rate slightly increases for both methods; as $N$ increases, the FN rate is greatly reduced.
In addition, from the augmented causal graph, we can identify causal directions  by the procedure in Section \ref{Sec:unify}. In this simulation, the whole causal DAG is correctly identified. However, with the original SGS method, we can only identify two causal directions: $5 \rightarrow 6$ and $2 \rightarrow 6$, and there are spurious edges $(V_1, V_4)$, $(V_1, V_5)$ and $(V_4, V_5)$. 

Furthermore, we visualized the nonstationarity of causal modules, $P(V_1)$, $(V2,V3)\rightarrow V_4$, and $V_3\rightarrow V_5$, by KNV (Algorithm 2).
We tried both the linear kernel and Gaussian kernel for $k^{(2)}$.
Figure \ref{simulated_visual_Sup} shows the first component of the extracted nonstationarity encapsulators $\hat{\lambda}_i$, $i=1,4,5$, corresponding to the three nonstationary causal models; see the blue solid lines.  Panels (a) and (b) correspond to the setting $w=5, N=600$ and $w=30, N=600$, respectively.  The red dashed lines show the changing parameters $f_1$, $f_2$, and $f_3$ in the respective causal models. Note that they have been rescaled to match with the nonstationarity encapsulators $\hat{\lambda}_i$.  We can see that KNV successfully recovers the variability in the causal models (as represented by changing parameters $f_1$, $f_2$, $f_3$, corresponding to changes in the causal strength or noise variance). In addition, the Gaussian kernel gives better results especially in the case where $w=30$.

\begin{figure}[ht]
\centering
\setlength{\abovecaptionskip}{-3pt}
\setlength{\belowcaptionskip}{-2pt}
 \centering
 \subfigure[]{
\includegraphics[width=.42\textwidth,height = 3.6cm,trim={1.4cm .2cm 2cm 0.2cm},clip]{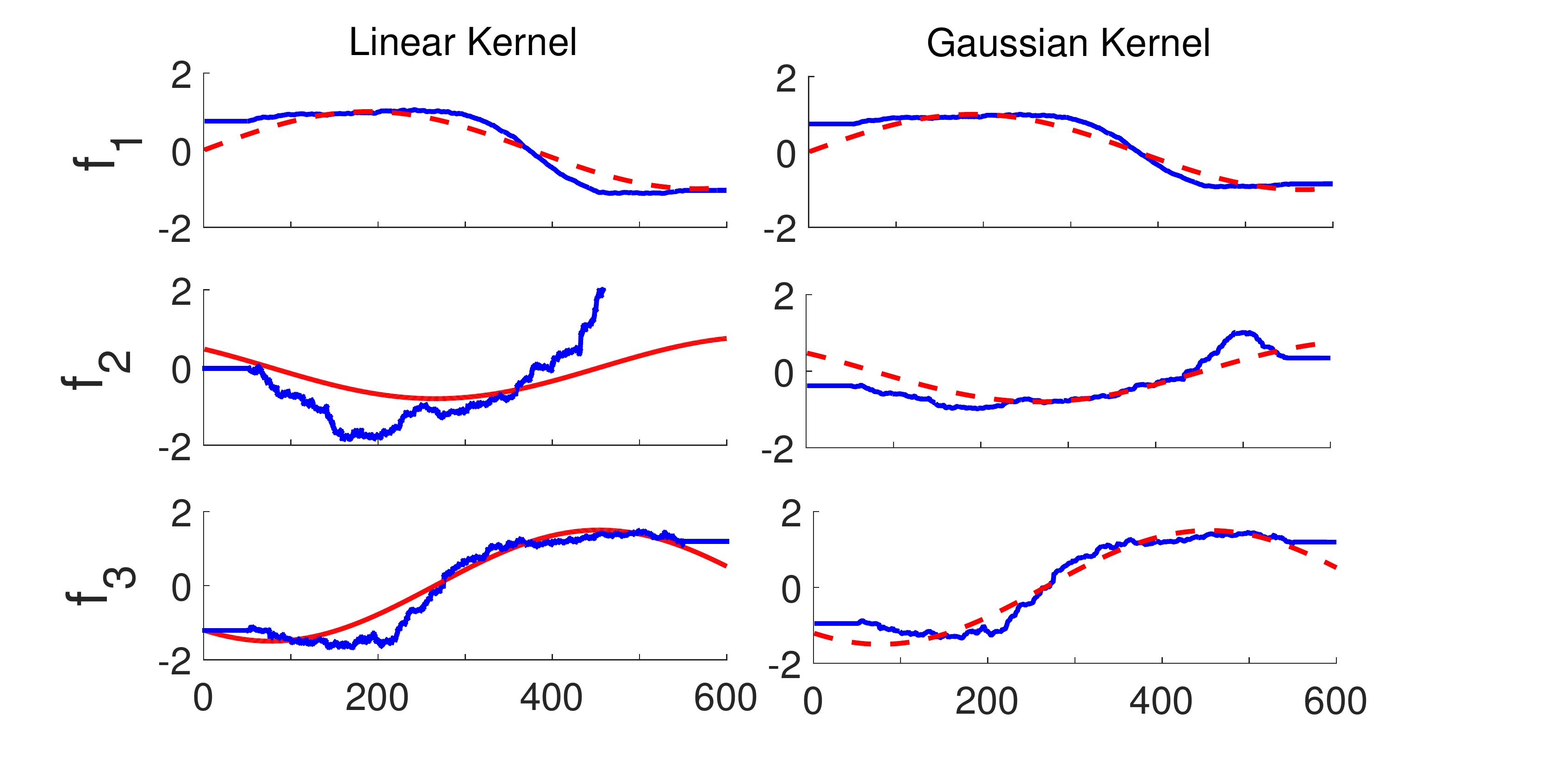}
}
\subfigure[]{
\includegraphics[width=.42\textwidth,height = 3.7cm,trim={1.4cm .2cm 2cm .2cm},clip]{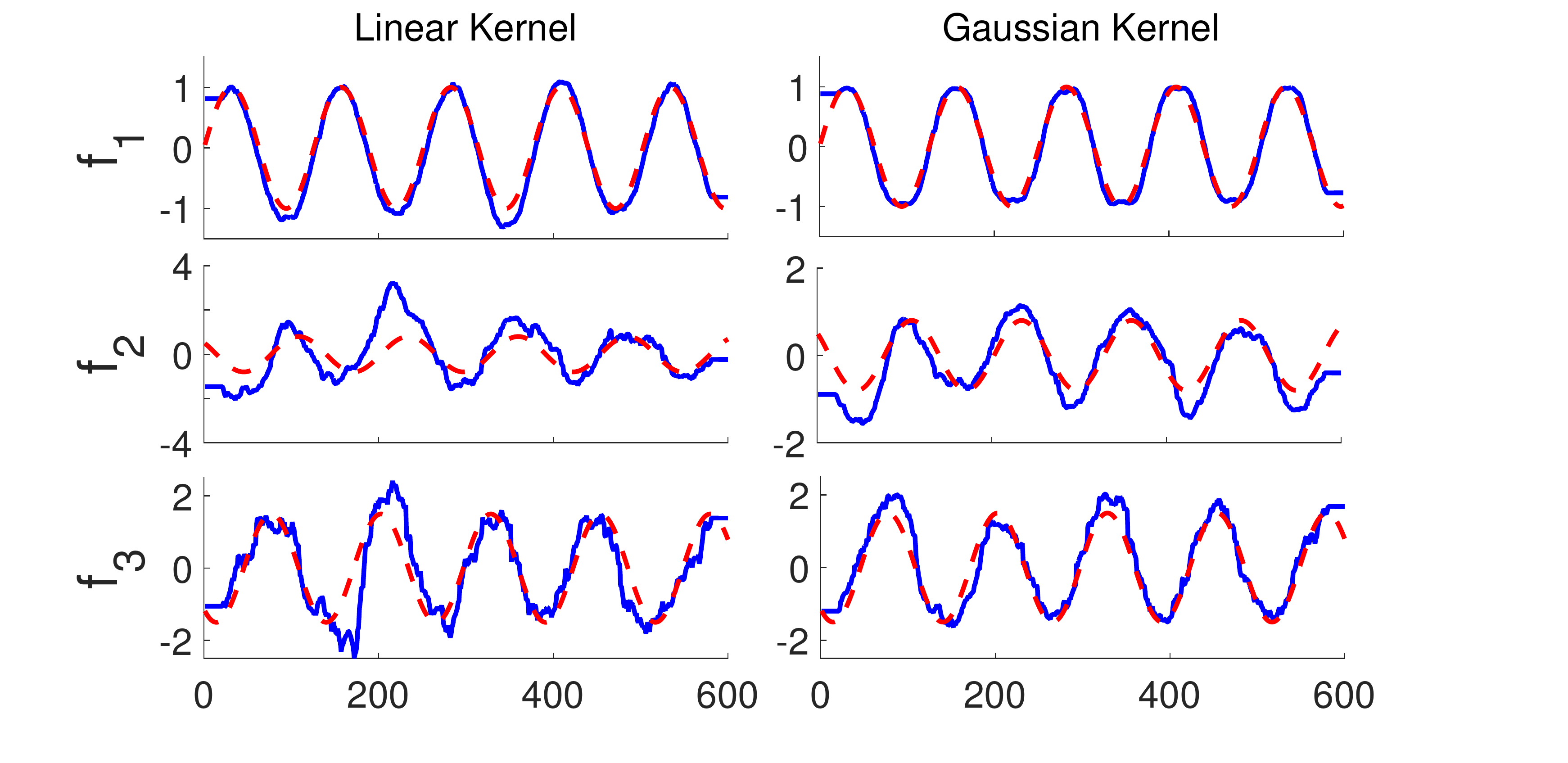}
}
 \caption{The estimated nonstationarity encapsulators given by KNV and the corresponding changing parameters $f_1$, $f_2$, and $f_3 $ in causal models $P(V_1)$, $(V2,V3)\rightarrow V_4$, and $V_3\rightarrow V_5$, tested on both linear kernel and Gaussian kernel for $k^{2}$ in KNV. The blue solid line represents the recovered signal, while the red dashed line represents the true signal. (a) For the setting $w=5$ and $N=600$. (b) $w=30$ and $N=600$.}
\label{simulated_visual_Sup}
\end{figure}



To summarize, we found that when there is only one changing parameter in the causal model $P(V_i\,|\,PA^i)$, which may be the linear coefficient, the mean of the noise, or its variance, with the Gaussian kernel for $k^{(2)}$, one component of $\lambda_i(C)$ is usually enough to capture the changes -- this component is close to a nonlinear transformation of the changing parameter, and its corresponding eigenvalue is at least five times bigger than the remaining ones. However, if the functional form of the causal model changes, say, if the SEM changes from a linear one to a quadratic one, more than one component of  $\lambda_i(C)$ has relatively large eigenvalues, and they jointly capture the change in $P(V_i\,|\,PA^i)$ (results are not included here).

\subsection{Experimental Results on Simulated fMRI}
In recent years the brain effective connectivity study from fMRI has received much attention. The fMRI experiments may last for a relatively long time period, during which the causal influences are likely to change along with certain unmeasured states (e.g., the attention) of the subject and ignoring the time-dependence may lead to spurious connections. Likewise, the causal influences may also vary as a function of the experimental condition (e.g., health, disease, and behavior) \cite{Chronnectome}.

Currently little is known for the causal connectivity in our brain, so firstly we applied our approach on simulated fMRI data which enables us to evaluate the robustness of our approach with known ground truth. We generated the simulated fMRI signal according to the DCM forward model \cite{DCM}. 

Fig. \ref{setting_dcm} shows a basic setting of the network topologies, where we modeled the external input $ u_1$ to the nodes as random square wave \cite{Smith11}, and the external input to the connections with different kinds of functions, e.g., exponential decay, square wave, and log functions. Since the study on how causal connections between brain regions are changing is very limited, we tried to represent them with different functions to model different possible scenarios. In addition, in practice we may analyze the fMRI signal concatenated from different scans (different subjects or different instruments), so in order to model this situation, we concatenated two generated BOLD signals to derive the final signal.

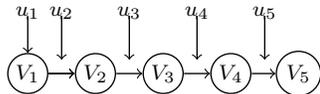
\begin{SCfigure}[\sidecaptionrelwidth][ht]
\setlength{\abovecaptionskip}{0pt}
\setlength{\belowcaptionskip}{0pt}
 \centering
 \begin{tikzpicture}[scale=.45, line width=0.5pt, inner sep=0.2mm, shorten >=.1pt, shorten <=.1pt]
 \draw (0, 0) node(1) [circle, draw] {{\footnotesize\,$V_1$\,}};
   \draw (2, 0) node(2) [circle, draw] {{\footnotesize\,$V_2$\,}};
 \draw (4, 0) node(3) [circle, draw] {{\footnotesize\,$V_3$\,}};
 \draw (6, 0) node(4) [circle, draw] {{\footnotesize\,$V_4$\,}};
 \draw (8, 0) node(5) [circle, draw] {{\footnotesize\;$V_5$\;}};
  \draw (0, 1.8) node(6) {{\footnotesize\;$u_1$\;}};
  \draw (1, 1.8) node(7) {{\footnotesize\;$u_2$\;}};
  \draw (3, 1.8) node(8) {{\footnotesize\;$u_3$\;}};
  \draw (5, 1.8) node(9) {{\footnotesize\;$u_4$\;}};
  \draw (7, 1.8) node(10) {{\footnotesize\;$u_5$\;}};
   \draw[->] (1) -- (2); 
   \draw[->] (2) -- (3);
   \draw[->] (3) -- (4);
   \draw[->] (4) -- (5);
   \draw[->] (1) -- (2);
\draw[->] (0,1.6) to[out=90,in=90](0,0.6);
\draw[->] (1,1.6) to[out=90,in=90](1,0.4);
\draw[->] (3,1.6) to[out=90,in=90](3,0.4);
\draw[->] (5,1.6) to[out=90,in=90](5,0.4);
\draw[->] (7,1.6) to[out=90,in=90](7,0.4);
 \end{tikzpicture}
 \caption{The basic setting of the network topologies.}
 \label{setting_dcm}
\end{SCfigure}

We tested our enhanced constraint-based method on 50 realizations, where the time information $ T $ is included into the system to capture smooth varying causal relations and the influences from smooth varying latent confounders. Fig. \ref{estimate_dcm1}(a) gives the False Positive (FP) rate and False Negative (FN) rate at significance level 0.03. We compared our enhanced constraint-based method with the original one (both with SGS search and KCI test), and we also compared with partial correlation test since it is widely used in fMRI analysis \cite{Smith11}. It is obvious that our approach greatly reduces the FP rate, that is, it effectively reduces spurious connections which are induced by the time-varying connections, while at the same time increases the FN rate in a reasonable range. The partial correlation test gives the worst results, with the FP rate 1 and the FN rate 0.1016 in a small-sample-size case. Since there is a certain amount of variation across realizations, we give a causal connection if it exists in more than $ 80\% $ of all the realizations. Fig. \ref{estimate_dcm1}(b-c) show the causal structures estimated by our approach and the original constraint-based method with KCI-test.  The partial correlation test produces fully connected graph.

\begin{figure}[ht]
\setlength{\abovecaptionskip}{0pt}
\setlength{\belowcaptionskip}{0pt}
\centering
\subfigure[Estimation error.]{
\setlength{\abovecaptionskip}{0pt}
\setlength{\belowcaptionskip} {0pt}
  \includegraphics[width=0.8\textwidth,height=4cm,trim={2cm .5cm 2cm .5cm},clip]{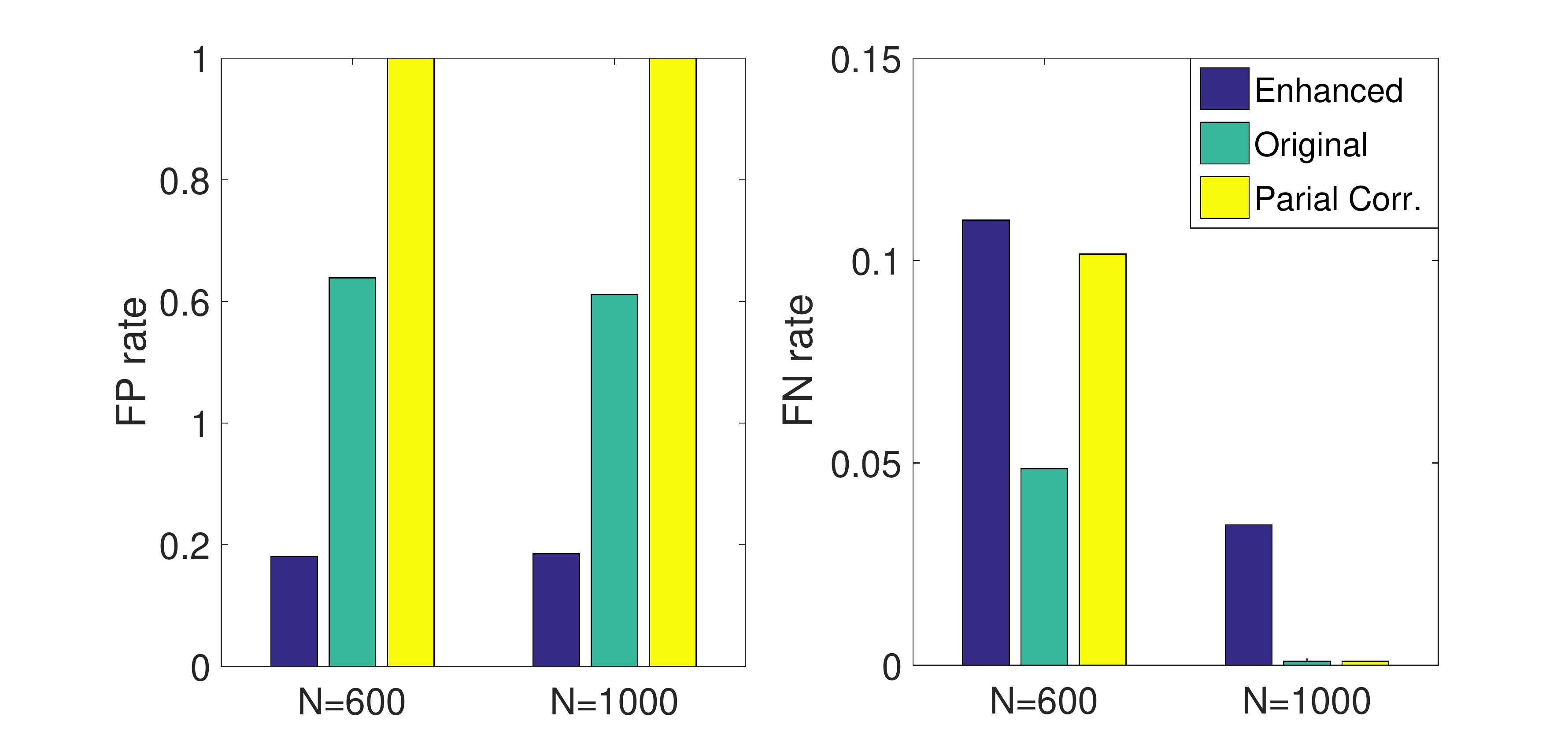}%
}
\\
 \begin{tikzpicture}[scale=.6, line width=0.5pt, inner sep=0.2mm, shorten >=.1pt, shorten <=.1pt]
 \draw (0, 0) node(1) [circle, draw] {{\footnotesize\,$V_1$\,}};
   \draw (2, 0) node(2) [circle, draw] {{\footnotesize\,$V_2$\,}};
 \draw (4, 0) node(3) [circle, draw] {{\footnotesize\,$V_3$\,}};
 \draw (6, 0) node(4) [circle, draw] {{\footnotesize\,$V_4$\,}};
 \draw (8, 0) node(5) [circle, draw] {{\footnotesize\;$V_5$\;}};
   \draw[-] (2) -- (3);
   \draw[-] (3) -- (4);
   \draw[-] (4) -- (5);
   \draw[->] (1) -- (2);
 \end{tikzpicture}
 \\
 ~~~~~~~(b) Estimated causal graph by the enhanced method~~~~~~~~~~\\
\vspace{0pt}
\hspace{-10pt}
  \begin{tikzpicture}[scale=.6, line width=0.5pt, inner sep=0.2mm, shorten >=.1pt, shorten <=.1pt]
  \draw (0, 0) node(1) [circle, draw] {{\footnotesize\,$V_1$\,}};
    \draw (1.5, 0) node(2) [circle, draw] {{\footnotesize\,$V_2$\,}};
  \draw (3, 0) node(3) [circle, draw] {{\footnotesize\,$V_3$\,}};
  \draw (4.5, 0) node(4) [circle, draw] {{\footnotesize\,$V_4$\,}};
  \draw (6, 0) node(5) [circle, draw] {{\footnotesize\;$V_5$\;}};
    \draw[-] (1) -- (2);
    \draw[-] (2) -- (3);
    \draw[-] (3) -- (4);
    \draw[-] (4) -- (5);
     \draw[-] (2) to[out=45,in=135] (5); 
    \draw[-] (3) to[out=-45,in=-135] (5);
    \end{tikzpicture}\\
    ~~~~~~~~~~~(c) Estimated graph by the original method~~~~~~~~~~~\\
\caption{ (a) The estimation error, FP rate  and FN rate, derived from our enhanced constraint-based method, the original constraint-based method with KCI and the partial correlation test. (b,c) The estimated causal graph by the our approach and the original constraint-based one.} 
\label{estimate_dcm1}
\end{figure} 

\section{Experiments on Real Data} \label{Sec:real}

\subsection{On Stock Returns}
We applied our method to daily returns of 10 major stocks in Hong Kong. The dataset is from the Yahoo finance database, containing daily dividend/split adjusted closing prices from 10/09/2006 to 08/09/2010. For the few days when the stock price is not available, a simple linear interpolation is used to estimate the price. Denoting the closing price of the $i_{th}$ stock on day $t$ by $P_{i,t}$, the corresponding return is calculated by $V_{i,t} = \frac{P_{i,t}-P_{i,t-1}}{P_{i,t-1}} $. The 10 stocks are Cheung Kong Holdings (1), Wharf (Holdings) Limited (2), HSBC Holdings plc (3), Hong Kong Electric Holdings Limited (4), Hang Seng Bank Ltd (5), Henderson Land Development Co. Limited (6), Sun Hung Kai Properties Limited (7), Swire Group (8), Cathay Pacific Airways Ltd (9) and Bank of China Hong Kong (Holdings) Ltd (10). $3, 5$ and $10$ belong to Hang Seng Finance Sub-index (HSF), $1, 8 $ and $9$ belong to Hang Seng Commerce \& Industry Sub-index (HSC), $2,6$ and $7$ belong to Hang Seng Properties Sub-index (HSP) and $4$ belongs to Hang Seng Utilities Sub-index (HSU). It is believed that during the financial crisis around 2008, the causal relations in Hong Kong stock market have changed.

Fig. \ref{stock} shows the estimated causal structure by our method, where $2, 3, 4, 5, 6$ and $7$ are found to be time-dependent as indicated by red cycles.  In contrast, the original constraint-based method produces four more edges, which are  (2,3), (3,6), (5,7) and (6,8).  We found that all time-dependent returns are in HSF, HSP, and HSU, which are directly affected by some unconsidered factors, e.g. policy changes.  Furthermore, we inferred the causal directions by the procedure given in Section \ref{Sec:unify}, and we found that all the inferred directions are reasonable. In particular, the within sub-index causal directions tend to satisfy the owner-member relationship. For example, $4 \rightarrow 1$ because $1$ partially owns $4$, and similarly for $5 \rightarrow 3$ and $9 \rightarrow 8$. Those stocks in HSF are the major causes to those in HSC and in HSP, and the stocks in HSP and HSU impact those in HSC. These causal relations match with the fact that financial institutions are in the leading position to impact other fields, and industries are usually affected by financial institutions, companies in properties, and companies in utilities. One exception is that, $10$, Bank of China Hong Kong in HSF, is affected by $2$ in HSP; it is perhaps because of Bank of China Hong Kong's close relation with Bank of China in mainland China. 

\begin{figure}
  \centering
  \setlength{\abovecaptionskip}{0pt}
\setlength{\belowcaptionskip}{0pt}
  \includegraphics[width=0.7\textwidth,trim={1.2cm 1cm 1.2cm 1cm},clip]{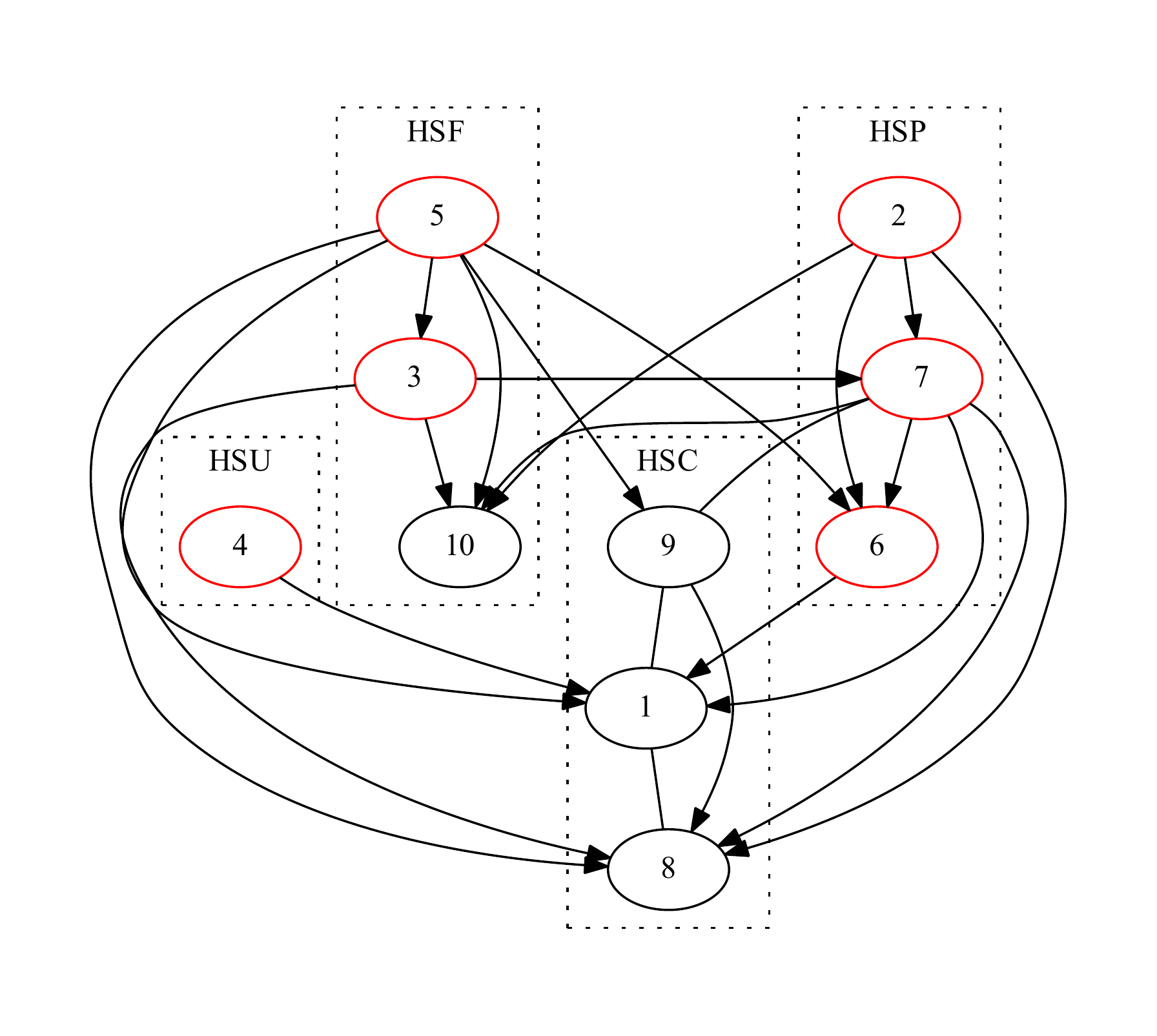}%
\caption{The estimated causal structure among the 10 stock returns. Red cycles indicate that the corresponding stock returns are time-dependent.  Our enhanced constraint-based method eliminated the edges for pairs (2,3), (5,7), (3,6) and (6,8), compared to the results by original SGS. }\label{stock}
\end{figure}
\begin{figure}
  \centering
  \setlength{\abovecaptionskip}{0pt}
\setlength{\belowcaptionskip}{0pt}
\includegraphics[width=0.265\textwidth,height = 1.2cm,trim={0.65cm 11.6cm 0.86cm 12cm},clip]{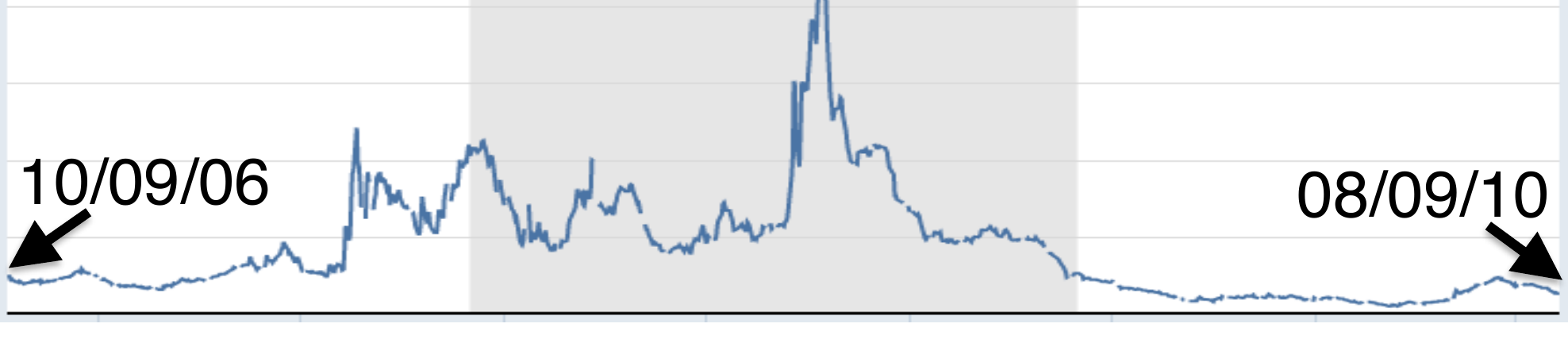}{\footnotesize~~~}~~~~~~~~~~~~~~~~~~~~~~~~~~~~~~~~~~~~~\\
  \includegraphics[width=0.8\textwidth,trim={1.2cm 0cm 1.2cm 0cm},clip]{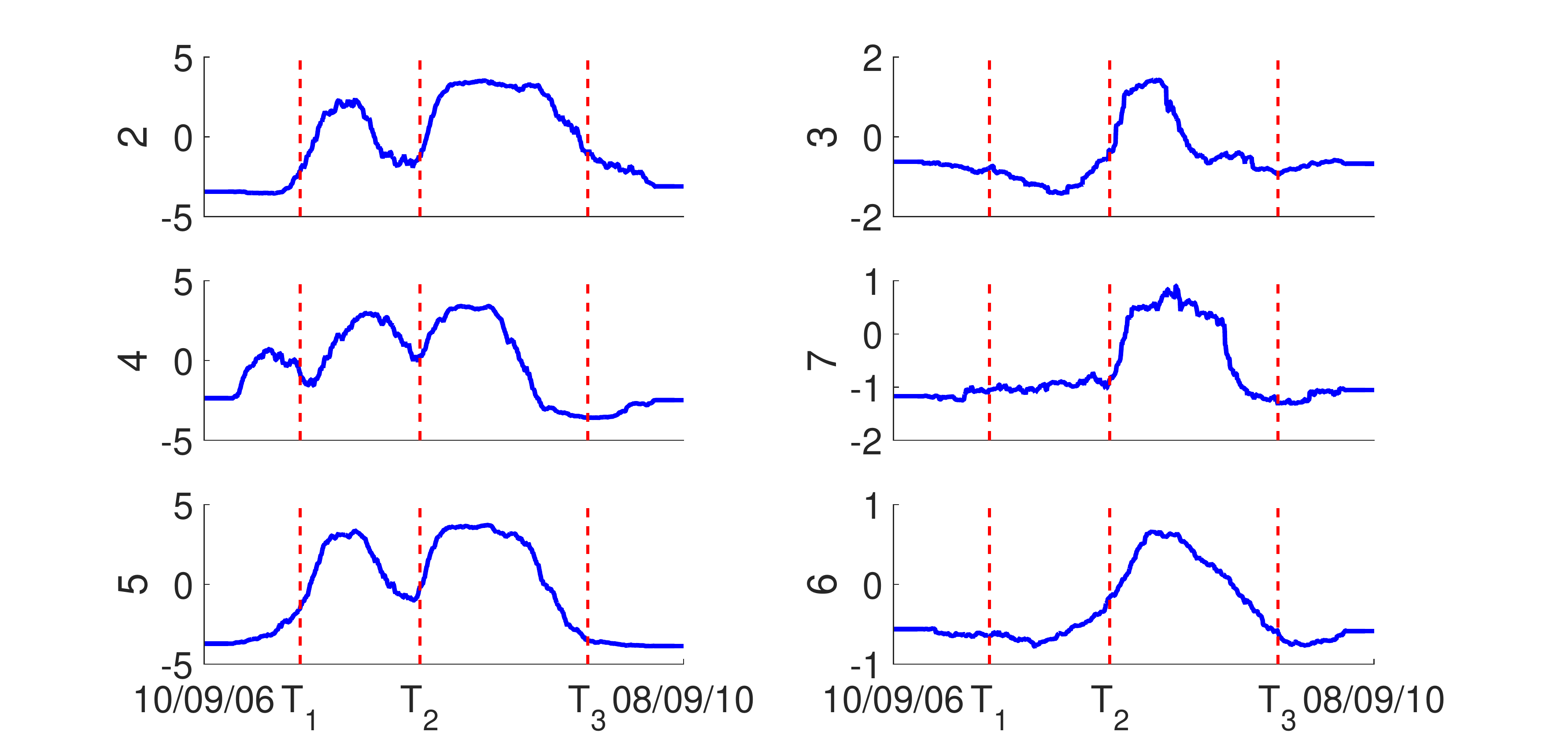}%
\caption{The visualized nonstationarity of causal modules of time-dependent stock returns as well as the curve of the TED spread over the same period. Top: Curve of the TED spread shown for comparison. 
Bottom: Visualized nonstationarity of causal modules for stocks $2,3,4,5,6$, and $7$, where $T_1$, $T_2$, and $T_3$ stand for  $07/16/2007$, $06/30/2008$, and $02/11/2009$, respectively. 
We can see that the nonstationary components of root causes, $2,4$, and $5$, share the similar variability with change points around $T_1$, $T_2$, and $T_3 $. The nonstationary components of $3$, $6$, and $7$ have change points only around $T_2$ and $T_3$.}\label{stock_visual}
\end{figure}

Figure \ref{stock_visual} (bottom panels) visualizes the nonstationarity of (changing) causal modules, for stocks $2,3,4,5,6$, and $7$.  We can see that the nonstationary encapsulators of root causes, $2,4$, and $5$, share a similar variability; 
the change points are around $T_1$ (07/16/2007), $T_2$ (06/30/2008), and $T_3$ (02/11/2009).
The nonstationary encapsulators of $3$, $6$, and $7$ have change points around $T_2$  (06/30/2008) and $T_3$ (02/11/2009), but without $T_1$, which means that at the beginning of financial crisis, these stocks were not directly affected by the change of external factors. These findings match with the critical time points of financial crisis around the year of $2008$.
The active phase of the crisis, which manifested as a liquidity crisis, could be dated from August, 2007,\footnote{See more information at \url{https://en.m.wikipedia.org/wiki/Financial_crisis_of_2007-08}.} around $T_1$.  The nonstationarity encapsulators, especially those of 2, 4, 5, and 3, seem to be consistent with the change of the TED spread,\footnote{See \url{https://en.m.wikipedia.org/wiki/TED_spread}.} which is an indicator of perceived credit risk in the general economy and shown in Figure \ref{stock_visual} (top panel) for comparison; 7 and 6 seem to be directly influenced by the change in the underlying unmeasured factor, which may be related to the credit risk, mainly from 2008.


\subsection{On fMRI Hippocampus}

This fMRI Hippocampus dataset \cite{Myconnectome} was recorded from six separate brain regions: perirhinal cortex (PRC), parahippocampal cortex (PHC), entorhinal cortex (ERC), subiculum (Sub), CA1, and CA3/Dentate Gyrus (CA3) in the resting states on the same person in 64 successive days. 
We are interested in investigating causal connections between these six regions in the resting states. The anatomical connections between them reported in the literature are shown in Fig. \ref{hipp}. We used the anatomical connections as a reference, because in theory a direct causal connection between two areas should not exist if there is no anatomical connection between them.
\begin{SCfigure}[\sidecaptionrelwidth][ht]
  \centering
  \setlength{\abovecaptionskip}{0pt}
\setlength{\belowcaptionskip}{0pt}
\vspace{-0pt}
  \includegraphics[width=0.3\textwidth,trim={1.3cm 1.3cm 1.3cm 1.3cm},clip]{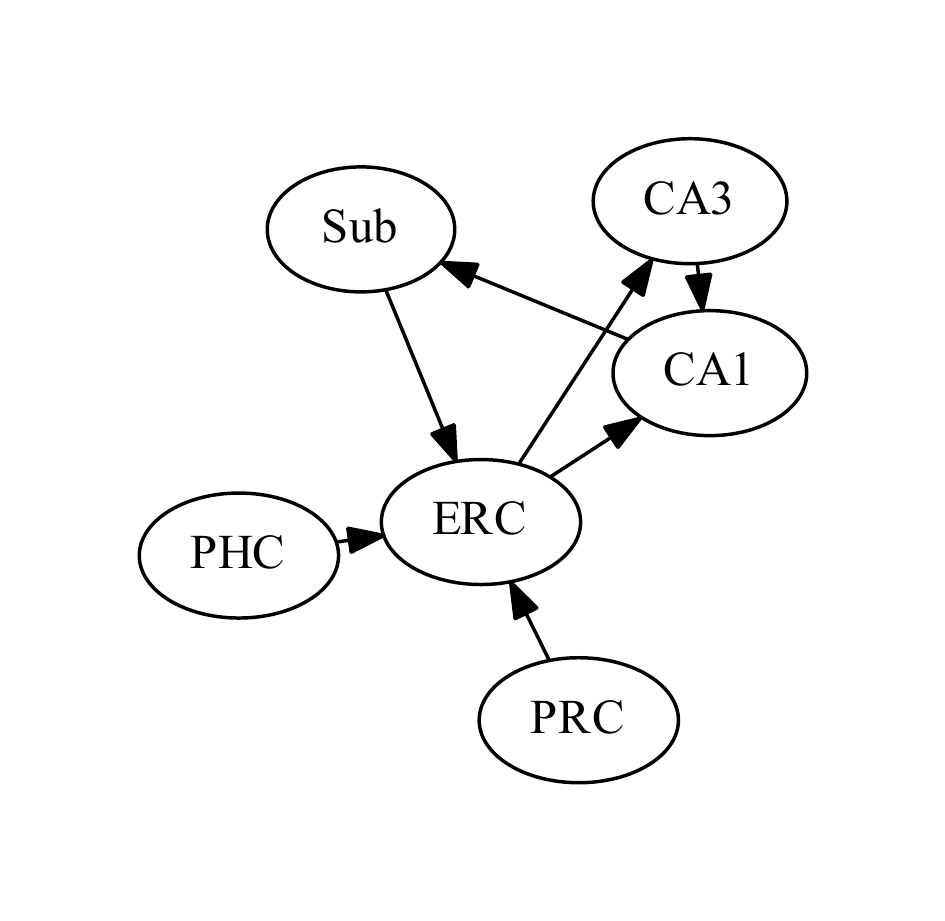}%
  \caption{The anatomical connections between the six separate brain regions.}\label{hipp} \vspace{-0.3cm}
\end{SCfigure}

We applied our enhanced constraint-based method on 10 successive days separately, with time information $T$ as an additional variable in the system. We assumed that the underlying causal graph is acyclic, although the anatomical structure gives cycles. We found that our method effectively reduces the FP rate, from $62.86 \%$ to $17.14 \%$, compared to the original constraint-based method with SGS search and KCI-test. Here we regard those connections that do not exist in the anatomical structure as spurious; however, with the lack of ground truth, we are not able to compare the FN rate. We found that the causal structure varies across days, but the connections between CA1 and CA3, and between CA1 and SUB are robust, which coincides with the current findings in neuroscience \cite{CA3_CA1}. In addition, on most datasets the causal graphs we derived are acyclic, which validates the use of constraint-based method. Furthermore, we applied the procedure in Section \ref{Sec:unify} to infer the causal directions. We successfully recovered the following causal directions: $\text{CA3} \rightarrow \text{CA1}$, $\text{CA1} \rightarrow \text{Sub}$, $\text{Sub} \rightarrow \text{ERC}$, $\text{ERC} \rightarrow \text{CA1}$ and $\text{PRC} \rightarrow \text{ERC}$, and the accuracy of direction determination is 85.71\% (we consider the anatomical connections, shown in Fig.~\ref{hipp}, as ground truth for the directions).

\subsection{On WiFi Dataset}

The WiFi dataset has been seen as a benchmark dataset to test the performance of domain adaptation algorithms. The indoor WiFi localization data can be easily outdated since the WiFi signal strength (features) may vary with time periods, devices, space and usage of the WiFi \cite{WIFIYQ}. Therefore, it is important to detect the domain-varying features in domain adaptation. In this dataset, the data were collected from three different time periods in the same locations.

We added the domain information ($D={1,2,3}$) as an additional variable in the causal system to capture the domain-varying features. Here we set the significance level as 0.05 and we found that only a small subsets of features (8/67) vary across domains, with feature index = \{1, 2, 3, 4, 5, 6, 12, 44\}, which provides benefits for further analysis in domain adaptation. We also found that compared to the original constraint-based method, our method gives much sparser connections between the features (the number of connections between the features is reduced from 52 to 26).

\section{Conclusion}

This paper is concerned with discovery and visualization of nonstationary models, where causal modules may change over time or across datasets.  We assume a weak causal sufficiency condition, which states that all confounders can be written as smooth functions of time or the domain index. We proposed (1) an enhanced constraint-based method for locating variables whose causal modules are nonstationary and estimating the skeleton of the causal structure over the observed variables, (2) a method for causal direction determination that takes advantage of the nonstationarity, and (3) a technique for visualizing nonstationary causal modules.

In this paper we only considered instantaneous or contemporaneous causal relations, as indicated by the assumption that the observed data are independently but not identically distributed; the strength (or model, or even existence) of the causal relations is allowed to change over time. We did not explicitly consider time-delayed causal relations and in particular did not engage autoregressive models. However, we note that it is natural to generalize our framework to incorporate time-delayed causal relations, just in the way that constraint-based causal discovery was adapted to handle time-series data (see, e.g.,~\cite{Chu08}).

There are several open questions we aim to answer in future work. First, in this paper we assumed that causal directions do not flip despite of nonstationarity. But what if some causal directions also change over time or across domains? Can we develop a general approach to detect causal direction changes? Second, to fully determine the causal structure, one might need to combine the proposed framework with other approaches, such as those based on restricted functional causal models. How can this be efficiently accomplished?  Third, the issue of distribution shift may decrease the power of statistical (conditional) independence tests. Is it possible to mitigate this problem?


\begin{thebibliography}{10}

\bibitem{DynamicC1}
M.~Havlicek, K.J. Friston, J.~Jan, M.~Brazdil, and V.D. Calhoun.
\newblock Dynamic modeling of neuronal responses in {fMRI} using cubature
  kalman filtering.
\newblock {\em Neuroimage}, 56:2109--2128, 2011.

\bibitem{Spirtes00}
P.~Spirtes, C.~Glymour, and R.~Scheines.
\newblock {\em Causation, Prediction, and Search}.
\newblock MIT Press, Cambridge, MA, 2nd edition, 2001.

\bibitem{Pearl00}
J.~Pearl.
\newblock {\em Causality: Models, Reasoning, and Inference}.
\newblock Cambridge University Press, Cambridge, 2000.

\bibitem{Chickering02}
D.~M. Chickering.
\newblock Optimal structure identification with greedy search.
\newblock {\em Journal of Machine Learning Research}, 3:507--554, 2003.

\bibitem{Heckerman95}
D.~Heckerman, D.~Geiger, and D.~M. Chickering.
\newblock Learning bayesian networks: The combination of knowledge and
  statistical data.
\newblock {\em Machine Learning}, 20:197--243, 1995.

\bibitem{Shimizu06}
S.~Shimizu, P.O. Hoyer, A.~Hyv{\"a}rinen, and A.J. Kerminen.
\newblock A linear non-{Gaussian} acyclic model for causal discovery.
\newblock {\em Journal of Machine Learning Research}, 7:2003--2030, 2006.

\bibitem{Hoyer09}
P.O. Hoyer, D.~Janzing, J.~Mooji, J.~Peters, and B.~Sch{\"o}lkopf.
\newblock Nonlinear causal discovery with additive noise models.
\newblock In {\em Advances in Neural Information Processing Systems 21},
  Vancouver, B.C., Canada, 2009.

\bibitem{Zhang09_additive}
K.~Zhang and A.~Hyv{\"a}rinen.
\newblock Acyclic causality discovery with additive noise: An
  information-theoretical perspective.
\newblock In {\em Proc. European Conference on Machine Learning and Principles
  and Practice of Knowledge Discovery in Databases (ECML PKDD) 2009}, Bled,
  Slovenia, 2009.

\bibitem{Zhang_UAI09}
K.~Zhang and A.~Hyv{\"a}rinen.
\newblock On the identifiability of the post-nonlinear causal model.
\newblock In {\em Proceedings of the 25th Conference on Uncertainty in
  Artificial Intelligence}, Montreal, Canada, 2009.

\bibitem{Mooij10_GPI}
J.~Mooij, O.~Stegle, D.~Janzing, K.~Zhang, and B.~Sch{\"o}lkopf.
\newblock Probabilistic latent variable models for distinguishing between cause
  and effect.
\newblock In {\em Advances in Neural Information Processing Systems 23 (NIPS
  2010)}, Curran, NY, USA, 2010.

\bibitem{Chronnectome}
V.~D. Calhoun, R.~Miller, G.~Pearlson, and T.~Adal.
\newblock The chronnectome: Time-varying connectivity networks as the next
  frontier in {fMRI} data discovery.
\newblock {\em Neuron}, 84(2):262--274, 2014.

\bibitem{Adam07}
R.~P. Adams and D.~J.~C. Mackay.
\newblock {\em Bayesian online change point detection}, 2007.
\newblock Technical report, University of Cambridge, Cambridge, UK. Preprint at
  http://arxiv.org/abs/0710.3742v1.

\bibitem{Talih05}
M.~Talih and N.~Hengartner.
\newblock Structural learning with time-varying components: Tracking the
  cross-section of financial time series.
\newblock {\em Journal of the Royal Statistical Society - Series B}, 67
  (3):321--341, 2005.

\bibitem{Kummerfeld13}
E.~Kummerfeld and D.~Danks.
\newblock Tracking time-varying graphical structure.
\newblock In {\em Advances in neural information processing systems 26}, La
  Jolla, CA, 2013.

\bibitem{Huang15}
B.~Huang, K.~Zhang, and B.~Sch{\"o}lkopf.
\newblock Identification of time-dependent causal model: A gaussian process
  treatment.
\newblock In {\em the 24th International Joint Conference on Artificial
  Intelligence, Machine Learning Track}, pages 3561--3568, Buenos, Argentina,
  2015.

\bibitem{Xing10}
E.~P. Xing, W.~Fu, and L.~Song.
\newblock A state-space mixed membership blockmodel for dynamic network
  tomography.
\newblock {\em Annals of Applied Statistics}, 4 (2):535--566, 2010.

\bibitem{Song09_DBN}
L.~Song, M.~Kolar, and E.~Xing.
\newblock Time-varying dynamic {Bayesian} networks.
\newblock In {\em Advances in neural information processing systems 23}, 2009.

\bibitem{Madiman08}
M.~Madiman.
\newblock On the entropy of sums.
\newblock In {\em Proceedings of IEEE Information Theory Workshop (ITW'08)},
  pages 303--307, 2008.

\bibitem{Zhang11_KCI}
K.~Zhang, J.~Peters, D.~Janzing, and B.~Sch{\"o}lkopf.
\newblock Kernel-based conditional independence test and application in causal
  discovery.
\newblock In {\em Proceedings of the 27th Conference on Uncertainty in
  Artificial Intelligence (UAI 2011)}, Barcelona, Spain, 2011.

\bibitem{Hoover90}
K.~Hoover.
\newblock The logic of causal inference.
\newblock {\em Economics and Philosophy}, 6:207--234, 1990.

\bibitem{Tian01}
J.~Tian and J.~Pearl.
\newblock Causal discovery from changes: a bayesian approach.
\newblock In {\em Proceedings of the 17th Conference on Uncertainty in
  Artificial Intelligence (UAI2001)}, pages 512--521, 2001.

\bibitem{effect13}
D.~Janzing, D.~Balduzzi, M.~Grosse-Wentrup, and B.~Sch{\"o}lkopf.
\newblock Quantifying causal influences.
\newblock {\em Ann. Statist.}, 41:2324--2358, 2013.

\bibitem{Scholkopf02}
B.~Sch{\"o}lkopf and A.~Smola.
\newblock {\em Learning with kernels}.
\newblock MIT Press, Cambridge, MA, 2002.

\bibitem{Song09_embedding}
L.~Song, J.~Huang, A.~Smola, and K.~Fukumizu.
\newblock Hilbert space embeddings of conditional distributions with
  applications to dynamical systems.
\newblock In {\em International Conference on Machine Learning (ICML 2009)},
  June 2009.

\bibitem{Fukumizu04dimensionalityreduction}
K.~Fukumizu, F.~R. Bach, M.~I. Jordan, and C.~Williams.
\newblock Dimensionality reduction for supervised learning with reproducing
  kernel {Hilbert} spaces.
\newblock {\em Journal of Machine Learning Research}, 5:73--99, 2004.

\bibitem{Scholkopf98}
B.~Sch{\"o}lkopf, A.~Smola, and K.~Muller.
\newblock Nonlinear component analysis as a kernel eigenvalue problem.
\newblock {\em Neural Computation}, 10:1299--1319, 1998.

\bibitem{Gretton07}
A.~Gretton, K.~Borgwardt, M.~Rasch, B.~Sch{\"o}lkopf, and A.~Smola.
\newblock A kernel method for the two-sample-problem.
\newblock In {\em NIPS 19}, pages 513--520, Cambridge, MA, 2007. MIT Press.

\bibitem{SGS93}
P.~Spirtes, C.~Glymour, and R.~Scheines.
\newblock {\em Causation, Prediction, and Search}.
\newblock Spring-Verlag Lectures in Statistics, 1993.

\bibitem{DCM}
K.~J. Friston, L.~Harrison, and W.~Penny.
\newblock Dynamic causal modelling.
\newblock {\em Neuroimage}, 19(4):1273--1302, 2003.

\bibitem{Smith11}
S.~M. Smith, K.~L. Miller, G.~Salimi-Khorshidi, M.~Webster, C.~F. Beckmann,
  T.~E. Nichols, and M.~W. Woolrich.
\newblock Network modelling methods for {fMRI}.
\newblock {\em Neuroimage}, 54(2):875--891, 2011.

\bibitem{Myconnectome}
R.~Poldrack.
\newblock \url{http://myconnectome.org/wp/}.

\bibitem{CA3_CA1}
D.~Song, MC. Hsiao, I.~Opris, RE. Hampson, VZ. Marmarelis, GA. Gerhardt, SA.
  Deadwyler, and TW. Berger.
\newblock Hippocampal microcircuits, functional connectivity, and prostheses.
\newblock {\em Recent Advances On the Modular Organization of the Cortex},
  pages 385--405, 2015.

\bibitem{WIFIYQ}
Qiang Yang, Sinno~Jialin Pan, and Vincent~Wenchen Zheng.
\newblock Estimating location using wi-fi.
\newblock 23(1):8--13, 2008.

\bibitem{Chu08}
T.~Chu and C.~Glymour.
\newblock Search for additive nonlinear time series causal models.
\newblock {\em Journal of Machine Learning Research}, 9:967--991, 2008.

\end{thebibliography}

\end{document}